\documentclass[twoside,11pt]{article}
\usepackage{blindtext}

\usepackage[abbrvbib, preprint]{jmlr2e}

\usepackage{amsmath}
\usepackage{mathrsfs} 
\usepackage{multirow}
\usepackage{enumitem}
\usepackage{bm}
\usepackage{hyperref}
\usepackage{float}
\usepackage{diagbox}
\usepackage{comment}
\usepackage{lscape}
\usepackage{booktabs} 
\usepackage{cool}
\usepackage[dvipsnames,table]{xcolor}
\usepackage{pifont} 
\usepackage{amsbsy} 
\usepackage{dutchcal} 
\usepackage{algorithm} 
\usepackage{algpseudocode}
\usepackage{soul}
\usepackage{mathtools} 
\usepackage{xargs}

\newcommand{\Og}{\Omega}
\newcommand{\fl}[2]{\frac{#1}{#2}}

\newcommand{\rmone}{\uppercase\expandafter{\romannumeral1}}
\newcommand{\rmtwo}{\uppercase\expandafter{\romannumeral2}}
\newcommand{\rmthree}{\uppercase\expandafter{\romannumeral3}}
\newcommand{\rmfour}{\uppercase\expandafter{\romannumeral4}}
\newcommand{\rmfive}{\uppercase\expandafter{\romannumeral5}}

\newcommand{\nn}{\nonumber}
\newcommand{\ap}{\alpha}

\newcommand{\Dt}{\Delta}

\newcommand{\be}{\begin{equation}}
\newcommand{\ee}{\end{equation}}
\newcommand{\ba}{\begin{array}}
\newcommand{\ea}{\end{array}}
\newcommand{\dee}{\;\mathrm{d}}
\newcommand{\bea}{\begin{eqnarray}}
\newcommand{\eea}{\end{eqnarray}}
\newcommand{\beas}{\begin{eqnarray*}}
\newcommand{\eeas}{\end{eqnarray*}}
\DeclareMathOperator*{\argmin}{\arg\!\min}

\newcommand{\R}{\mathbb{R}}
\newcommand{\expect}{\mathbb{E}}
\newcommand{\inputspace}{\mathcal{U}}
\newcommand{\outputspace}{\mathcal{V}}
\newcommand{\inputdimension}{{d_{\inputspace}}}
\newcommand{\outputdimension}{{d_{\outputspace}}}
\newcommand{\inputdomain}{{\Omega_{\inputspace}}}
\newcommand{\outputdomain}{{\Omega_{\outputspace}}}
\newcommand{\inputdiscdimension}{m}
\newcommand{\outputdiscdimension}{n}
\newcommand{\inputdiscretization}{D_{\inputspace}}
\newcommand{\outputdiscretization}{D_{\outputspace}}
\newcommand{\inputdiscspace}{\mathscr{D}(\inputspace)}
\newcommand{\outputdiscspace}{\mathscr{D}(\outputspace)}
\newcommand{\opdiscspace}{\mathscr{D}(\inputspace, \outputspace)}

\renewcommand{\vec}[1]{\mathbf{#1}}
\newcommand{\trueop}{G^\dagger}
\newcommand{\modelparam}{\theta}
\newcommand{\model}[1]{G_{#1}}  
\newcommand{\ourmodel}{G_\textnormal{ours}}
\newcommand{\numericalmodelnaked}[1]{\widetilde{G}_{#1}}
\newcommand{\numericalmodel}[5][]{\numericalmodelnaked{#2}#1(#3, #4; #5#1)}  
\newcommand{\dataelem}[2]{{#1}^{(#2)}} 
\newcommand{\inputlatentdim}{p}
\newcommand{\outputlatentdim}{q}
\newcommand{\encoder}{E}
\newcommand{\approximator}{A}
\newcommand{\reconstructor}{R}
\newcommand{\encoderparams}{\theta_E}
\newcommand{\approxparams}{\theta_A}
\newcommand{\reconstructparams}{\theta_R}
\newcommand{\encodernet}{\Phi^E}
\newcommand{\approxnet}{\Phi^A}
\newcommand{\reconstructnet}{\Phi^R}

\definecolor{MediumBlue}{RGB}{0,70,200}

\newcommand{\cn}{\color{black}}

\usepackage{lastpage}

\ShortHeadings{Discretization-independent Operator Learning for PDEs}{Hauck and Zhang}
\firstpageno{1}

\begin{document}

\title{Discretization-independent operator learning for partial differential equations}

\author{\name Jacob Hauck\email jacobhauck@mst.edu \\
       \addr Department of Mathematics and Statistics\\
       Missouri University of Science and Technology\\
       Rolla, MO 65409, USA
       \AND
       \name Yanzhi Zhang \email zhangyanz@umsystem.edu \\
        \addr Department of Mathematics and Statistics\\
       Missouri University of Science and Technology\\
       Rolla, MO 65409, USA}


\maketitle

\begin{abstract}%
We develop a new and general encode-approximate-reconstruct operator learning model that leverages learned neural representations of bases for input and output function distributions. We introduce the concepts of \textit{numerical operator learning} and \textit{discretization independence}, which clarify the relationship between theoretical formulations and practical realizations of operator learning models. Our model is discretization-independent,  making it particularly effective for {\cn multiresolution} learning. We establish theoretical approximation guarantees, demonstrating uniform universal approximation under strong assumptions on the input functions and statistical approximation under weaker conditions. 
To our knowledge, this is the first comprehensive study that investigates how discretization independence enables robust and efficient {\cn multiresolution} operator learning. We validate our method through extensive numerical experiments involving both local and nonlocal PDEs, including time-independent and time-dependent problems.  The results show that {\cn multiresolution} training significantly improves accuracy and computational efficiency. Moreover, {\cn multiresolution} training further enhances empirical discretization independence.
\end{abstract}

\begin{keywords}
  Operator learning, discretization independence, {\cn multiresolution learning,} reduced-order modeling, discretization bias, performance gap
\end{keywords}

\section{Introduction}
\label{section1}

Operator learning is a relatively new form of machine learning in which one attempts to learn an operator, in the functional analysis sense, instead of a single function \citep{tanyu_2023, kovachki_2024a}. 
Such models are usually applied to learning function-to-function mappings in the context of partial differential equations (PDEs), most often the solution operators of PDEs, which map an input function, like an initial condition or coefficient function, to the corresponding solution of the PDE \citep{kovachki_2024}. Operator learning has been applied to many scientific and engineering problems, for example to fluid flow in porous media 
and turbulent flow modeling in fluid mechanics \citep{li_2022c}, to crack propagation in material science \citep{goswami_2022}, to climate modeling \citep{kissas_2022}, and to multiphase flow in geophysics \citep{wen_2022}. Operator learning for solving complex problems is increasing in popularity primarily due to its potential to provide significantly more efficient surrogate solvers that rely only on data (meaning that no knowledge of the underlying system is required) \citep{li_2021a}. 

To be considered true approximations of an operator, models in operator learning should, like traditional numerical methods, have the ability to produce accurate results independent of the discretization used to represent the input and output functions \citep{li_2021a}. We will call this property \emph{discretization independence}.\footnote{Although \citet{kovachki_2024} and \cite{bartolucci_2023} have proposed some ideas to define discretization independence mathematically, there is still a lack of consensus. Nonetheless, it is generally agreed that some form of discretization independence distinguishes operator learning from other machine learning methods.}
An important consequence of discretization independence of operator learning models is that they can handle naturally {\cn multiresolution} problems, which are characterized by the presence of data with {\cn multiple representations of the true underlying functions, called \emph{resolutions} in reference to the fact that some representations resolve the underlying data more accurately than others.}
{\cn Data resolution in operator learning} is important because {\cn there is a trade-off between model accuracy and training cost: generally, the higher the data resolution, the higher the model accuracy, but also the higher the training cost.} 
{\cn \emph{Multiresolution learning} reduces the cost of model training by mixing low- and high-resolution training data, using as much low-resolution data as possible without sacrificing accuracy \citep{sakarvadia_2026}, which is particularly important since operator learning often requires large datasets to be effective. 
However, optimizing this cost–accuracy tradeoff is largely unstudied in operator learning. Recent work has suggested that earlier claims of zero-shot super-resolution, the ability to train models only on low-resolution data that generalize well to arbitrary resolutions \citep{li_2021a}, may not in fact be generally true \citep{sakarvadia_2026,subedi_2026}. 
We aim to improve the understanding of multiresolution learning in this work.}

Existing operator learning methods can be broadly classified into three groups according to their overall structure: neural operators \citep{li_2020, li_2021a}, transformers \citep{cao_2021, li_2023}, and encode-approximate-reconstruct models \citep{kovachki_2024a,hao_2023}. 
The encode-approximate-reconstruct framework has two main advantages over other approaches: simple structure and well-understood approximation theory \citep{lanthaler_2022}. These methods, however, often fail to be discretization independent because some architectural choices depend on a specific representation of the input or output function \citep{bhattacharya_2021, hesthaven_2018, lu_2021}. For example, DeepONet \citep{lu_2021} requires the input function to be sampled on a fixed set of sensor points; more detailed discussion on discretization independence, or lack thereof, for encode-approximate-reconstruct methods is given in Section \ref{sec:methods:discretization-and-comparison}. {\cn Consequently, these methods currently struggle to handle multiresolution learning. }
So far a limited number of other methods have attempted to use encode-approximate-reconstruct models for {\cn multifidelity learning \citep{howard_2023, lu_2022a, de_2023}, which can be adapted to a multiresolution setting}, but these methods require separate submodels to handle {\cn different resolutions}, as the underlying models used are not discretization independent.
Related ideas have been explored for neural operators \citep{kovachki_2024} and transformers \citep{cao_2021, fonseca_2023}, where the focus is on discretization invariance across different sampling meshes \citep{li_2023, li_2021a}. We remark that discretization invariance can be viewed as a special case of a broader discretization-independent perspective. 

In this work, we develop a new and general encode-approximate-reconstruct operator learning model that is discretization-independent  and inherently enables  {\cn multiresolution} learning (see Figure~\ref{fig:model-flow-chart}).  
Our main contributions can be summarized as follows: 
\begin{itemize}\itemsep -2pt
\item We develop the new concepts of \emph{numerical operator learning} and \emph{discretization independence}. These concepts clarify the relationship between theoretical and practical realizations of operator learning and provide a novel qualitative means of comparing operator learning models in terms of discretization independence, which we demonstrate on several popular models.
\item We propose a novel operator learning architecture that learns neural representations of bases for input and output function distributions. 
We prove theoretical approximation guarantees for our method, demonstrating universal approximation in a uniform sense, under strong assumptions on the input functions, and in a statistical sense, under weaker assumptions. Our method has significantly better discretization independence compared to existing approaches. 
\item Our method enables the seamless integration of data representations  with variable {\cn resolution} through the explicit handling of different discretizations, making it particularly effective for {\cn multiresolution} learning. We conduct the first detailed study on how discretization independence  facilitates  {\cn multiresolution} operator learning and validate it through extensive experiments.
\end{itemize}
\vspace{-2mm}
Our experiments verify the discretization independence of our method.  
Furthermore, we demonstrate that {\cn multiresolution} training can reduce some types of error caused by  low-{\cn resolution} training data, enabling faster training with  lower-{\cn resolution}.
Moreover, our experiments show that {\cn multiresolution} training can further improve empirical discretization independence.

\begin{figure}[htb!]
    \centering
    \includegraphics[width=\linewidth]{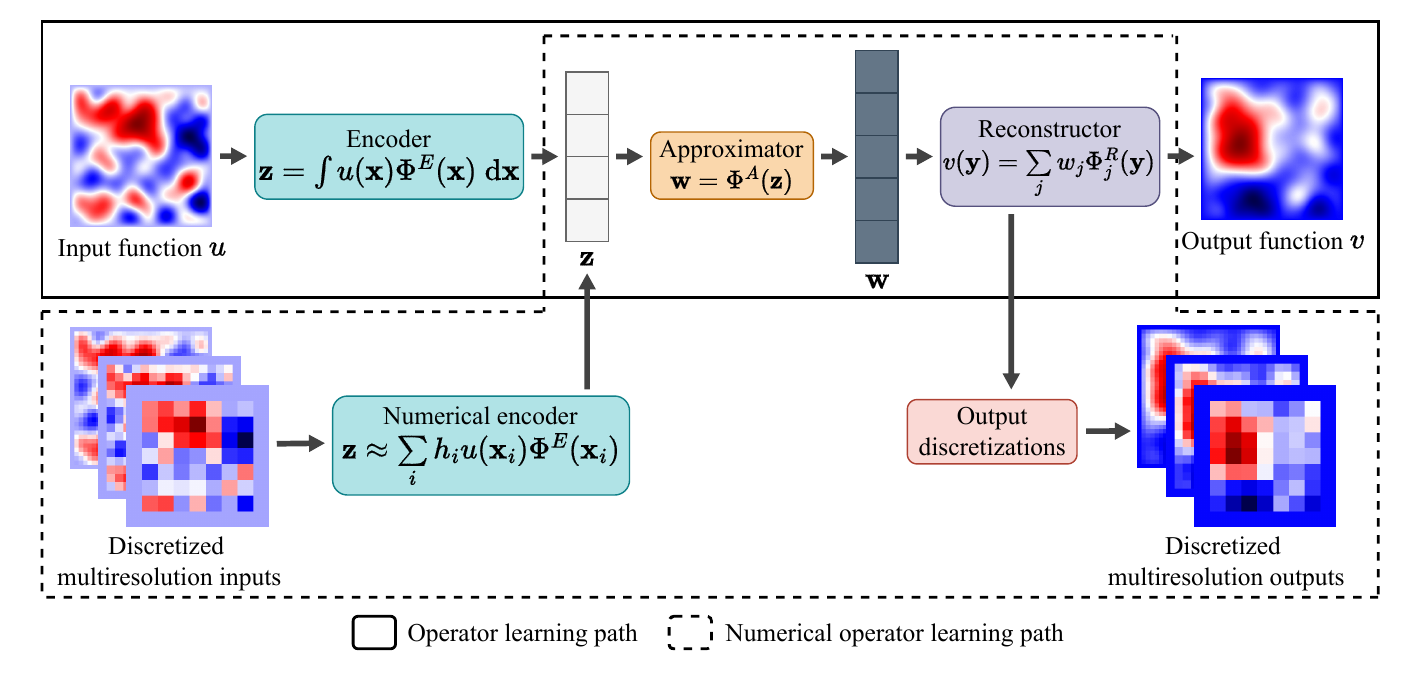}
    \vspace{-6mm}\caption{Diagram of the proposed operator learning model.}
    \label{fig:model-flow-chart}
    \vspace{-2mm}
\end{figure}

This paper is organized as follows. In Section \ref{section:background} we introduce operator learning and present a more general concept of discretization independence. In Section \ref{section:model} we describe our new method, demonstrate its discretization independence properties, prove approximation theorems, and show that other encode-approximate-reconstruct  models have limited  discretization independence in comparison. In Section \ref{section:experiments} we perform numerical experiments that verify discretization independence of our method in practice and also provide insights into issues with   single-{\cn resolution} learning and benefits of {\cn multiresolution} learning using our model. Finally, in Section \ref{section:conclusion} we provide some concluding remarks and discussion. 

\section{Operator learning}
\label{section:background}
\newcommand{\risk}{\mathcal{R}}
\newcommand{\disc}{\mathcal{D}}

In this section we establish the main concepts of operator learning and discuss its fundamental properties.  
Specifically, we introduce our novel concept of \emph{numerical operator learning}. 
We also provide a brief review and discussion of the key ideas and limitations of existing operator learning approaches.

\subsection{Numerical operator learning}
\label{sec:background:operator-learning}

We start by providing a precise definition of operator learning (in the supervised learning paradigm), which will require fairly complex notation; see Appendix \ref{sec:appendix:notation} for a comprehensive list. 
Let $\inputspace$ and $\outputspace$ be infinite-dimensional vector spaces, and let $\trueop : \inputspace \to \outputspace$ denote an operator mapping the input space $\inputspace$ to the output space $\outputspace$.
Suppose that we can collect data $\dataelem{u}{i} \in \inputspace$ and $\dataelem{v}{i} \in \outputspace$ that satisfy $\dataelem{v}{i} = \trueop(\dataelem{u}{i})$ for $i =1,2,\dots, N$.  
The objective of operator learning is to learn an approximation of $\trueop$ from the dataset $\{( \dataelem{u}{i}, \dataelem{v}{i})\}_{i=1}^N$. 
The data points $\dataelem{u}{i} \in \inputspace$ and $\dataelem{v}{i} \in \outputspace$ are, in general, infinite-dimensional.

By contrast, {\cn in traditional supervised learning approaches, such as linear regression, support vector machines, and neural networks,}
the data points consist of finite-dimensional vectors, and one uses the data to approximate a function $g : \R^m \to \R^n$, rather than an operator.
{\cn In operator learning, the domain and codomain of the target mapping} $g$ are allowed to be more general vector spaces than $\R^m$ and $\R^n$, mirroring the generalization of a linear transformation in linear algebra to a linear operator in functional analysis.
We now present a formal definition of operator learning within the supervised learning paradigm.

\smallskip
\begin{definition}[Supervised Operator Learning]  
\label{def:operator_learning} 
Let $\inputspace$ and $\outputspace$ be (infinite-dimensional) vector spaces, and let $\trueop : \inputspace \to \outputspace$ be an operator. An \emph{operator learning model} is a parametric operator $\model{\modelparam}: \inputspace \to \outputspace$ depending on the model parameter vector $\modelparam \in \R^P$.
Given a \emph{dataset} $\{(\dataelem{u}{i}, \dataelem{v}{i})\}_{i=1}^N$, where $\dataelem{u}{i} \in \inputspace$, $\dataelem{v}{i} \in \outputspace$, and $\dataelem{v}{i} = \trueop(\dataelem{u}{i})$, the model $\model{\modelparam}$ is \emph{trained} to approximate $\trueop$ by minimizing the \textnormal{empirical risk} $\risk : \R^P \to \R$ with respect to $\theta$, i.e.,
\begin{equation}
    \label{eq:empirical_risk_minimization}
    \modelparam^* = \argmin_{\modelparam} \risk(\modelparam), \qquad\mbox{with} \ \  \risk(\modelparam) = \frac{1}{N}\sum_{i=1}^NL\big(\model{\modelparam}\big(\dataelem{u}{i}\big), \dataelem{v}{i}\big),
\end{equation} 
 where $L : \outputspace \times \outputspace\to \R$, called the \textnormal{loss functional}, is a functional measuring the error between $\model{\modelparam}\big(\dataelem{u}{i}\big)$ and $v^{(i)} = \trueop(u^{(i)})$.
\end{definition}

Theoretically, this definition is little more than a general formulation of parametric, supervised learning via empirical risk minimization. Computationally, however, a data point $(\dataelem{u}{i}, \dataelem{v}{i})$ must be represented by a pair of finite-dimensional vectors $(\dataelem{\vec{u}}{i},\dataelem{\vec{v}}{i})$, where $\dataelem{\vec{u}}{i} \in \R^{m_i}$ and $\dataelem{\vec{v}}{i} \in \R^{n_i}$ are discrete representations of the true, infinite-dimensional data $\dataelem{u}{i} \in \inputspace$ and $\dataelem{v}{i} \in \outputspace$. 
Note that $m_i \ne n_i$, so that the representations of $\dataelem{u}{i}$ and $\dataelem{v}{i}$ are not necessarily related; moreover, the representations depend on $i$, so that each data point may potentially be represented in a different way.
Hence, a practical operator learning model must be able to handle different representations of its inputs and outputs. 
We propose a framework called \emph{numerical operator learning}, in which we account for computational representations explicitly.

The process of representing the infinite-dimensional data $\dataelem{u}{i}$ and $\dataelem{v}{i}$ by finite-dimensional vectors, that is, discretization, is described by mappings called \emph{discretization operators}, which we denote by $\inputdiscretization : \inputspace \to \R^\inputdiscdimension$ and $\outputdiscretization : \outputspace \to \R^\outputdiscdimension$.
The pair $(\inputdiscretization, \outputdiscretization)$ defines an \emph{operator discretization}. We denote the set of all operator discretizations by $\opdiscspace = \inputdiscspace \times \outputdiscspace$, where $\inputdiscspace$ and $\outputdiscspace$ represent the sets of all input and output discretizations, respectively. 
Then $(\dataelem{\vec{u}}{i},  \dataelem{\vec{v}}{i}) := (\inputdiscretization^i(u^{(i)}), \outputdiscretization^i(v^{(i)}))$ denotes the discrete representation of $(u^{(i)}, v^{(i)})$.  
The approximation of operator learning using these representations consists of approximating the discrete representation $\dataelem{\vec{v}}{i} = \outputdiscretization^i(\model{\modelparam}(\dataelem{u}{i}))$ of the output of the operator learning model in terms of the discrete representation $\dataelem{\vec{u}}{i}$ of the input.
One also needs to approximate the loss functional $L(\model{\modelparam} (\dataelem{u}{i}), \dataelem{v}{i})$ in terms of $\dataelem{\vec{u}}{i}$ and $\dataelem{\vec{v}}{i}$.
This way, the empirical risk minimization process can be performed using an approximation of the true empirical risk that depends only on the discrete representations of the input and output. 
This leads us to the definition of numerical operator learning.
\begin{definition}[Supervised Numerical Operator Learning]
 \label{def:numerical_operator_learning}
        Let $\inputspace$ and $\outputspace$ be (infinite-dimensional) vector spaces, and let $\trueop : \inputspace \to \outputspace$ be an operator. Let $\mathcal{D} = \{(\inputdiscretization^\alpha, \outputdiscretization^\alpha)\}_\alpha$ be a set of \emph{discretizations} indexed by a parameter $\alpha$, so that $\inputdiscretization^\alpha : \inputspace \to R^{\inputdiscdimension_\alpha}$, and $\outputdiscretization^\alpha : \outputspace \to \R^{\outputdiscdimension_\alpha}$. 
 A \textnormal{numerical operator learning model} is a parametric mapping $\numericalmodel{\modelparam}{ \inputdiscretization}{\outputdiscretization}{\cdot}:\R^\inputdiscdimension \to \R^\outputdiscdimension$,  provided that there is an operator learning model $\model{\modelparam}$ and a numerical realization $\ell(\outputdiscretization;\cdot,\cdot) : \R^\outputdiscdimension\times \R^\outputdiscdimension \to \R$ of the loss functional $L$ so that
\begin{equation}
    \label{eq:numerical_operator_learning}
    \outputdiscretization(\model{\modelparam}(u)) \approx \numericalmodel{\modelparam}{\inputdiscretization}{\outputdiscretization}{\inputdiscretization(u)}, \qquad \ell(\outputdiscretization; \outputdiscretization(v), \outputdiscretization(v')) \approx L(v, v'),
\end{equation}
where $v$, $v' \in \outputspace$.  With this, one may approximate the empirical risk minimization in \eqref{eq:empirical_risk_minimization} by  minimizing the \textit{numerical empirical risk} $\widetilde{\risk}$ to find a numerical optimal parameter $\widetilde{\modelparam}^* \in \R^P$:
\begin{equation}
    \label{eq:numerical_empirical_risk_minimization}
    \widetilde{\modelparam}^* = \argmin_\modelparam \widetilde{\risk}(\modelparam), \qquad\mbox{with} \ \   \widetilde{\risk}(\modelparam) = \frac{1}{N}\sum_{i=1}^N \ell\big(\outputdiscretization^i; \numericalmodel{\modelparam} {\inputdiscretization^i}{\outputdiscretization^i}{\dataelem{\vec{u}}{i}}, \dataelem{\vec{v}}{i}\big).
\end{equation}    
\vspace{-5mm}
\end{definition}
\noindent That is, $\numericalmodel{\modelparam}{\inputdiscretization}{\outputdiscretization}{\cdot}$ can be viewed as a numerical realization of the operator model $\model{\modelparam}$. 

We remark that for any given $i$, and possibly even all $i$, we could have $\inputdiscretization^i \ne \outputdiscretization^i$. Furthermore, for any $i\ne i'$, it is possible that $\inputdiscretization^i \ne \inputdiscretization^{i'}$ or $\outputdiscretization^i \ne \outputdiscretization^{i'}$. Thus, the number and meaning of the components of $\dataelem{\vec{u}}{i}$ and $\dataelem{\vec{v}}{i}$, the actual data used computationally,  depend on $i$. This is the key difference between numerical operator learning and {\cn  traditional supervised learning methods for finite-dimensional data}; a numerical operator learning model must handle $\dataelem{\vec{u}}{i}$ and $\dataelem{\vec{v}}{i}$ so that it can adapt to different $\inputdiscretization^i$ and $\outputdiscretization^i$.

\medskip
\noindent{\bf Discretization independence }
The discretizations $\inputdiscretization$ and $\outputdiscretization$  play a significant role in the design of operator learning models, and vice versa. 
In fact, every practical method of numerical operator learning  model works only under some more or less restrictive assumptions about the types of discretizations being used. 
Suppose a numerical operator learning model $\numericalmodelnaked{\modelparam}$ provides a valid approximation of an underlying operator model  for any operator discretization ${(\inputdiscretization, \outputdiscretization) \in \disc \subseteq \opdiscspace}$. We call the set $\disc$ the \emph{effective discretization set} of the operator learning model.  
The set $\disc$ consists precisely of those pairs of discretizations under which the approximations in \eqref{eq:numerical_operator_learning} hold. 
There is a lot of freedom in choosing the set $\disc$ of effective discretizations, which allows for the possibility of a numerical operator learning method that works for only one discretization or for only a limited set of discretizations. For example, the Fourier Neural Operator (FNO) \citep{li_2021a} relies on using the Fast Fourier Transform (FFT) to approximate efficiently the convolution integrals in the model. The FFT algorithm only works when the input function is discretized by sampling on a uniform grid; as a result, FNO can be applied only when the input and output functions are discretized by sampling on a uniform grid, which can be expressed as a constraint on FNO's effective discretization set $\mathcal{D}$. 

In this work, we will refer to the (vaguely defined) requirements of \eqref{eq:numerical_operator_learning} as \textit{discretization independence} to avoid confusion with either of the more specific ideas proposed. Discretization independence is the defining characteristic of numerical operator learning models, so one could compare models based on how well they achieve this property.  A model is more discretization independent than another if its effective discretization set is larger (see Appendix \ref{sec:appendix:compare} for a precise definition). 
Thus, we have a systematic means of comparing discretization independence for different methods. We will construct the effective discretization sets for the method that we developed as well as for a few other similar methods (see Section \ref{sec:methods:discretization-and-comparison}); as we will see, our method has a larger effective discretization set. Naturally, all else being equal, a model with greater discretization independence is more desirable.

We note that the nature of the approximations in \eqref{eq:numerical_operator_learning} is left vague on purpose, as it is not fully decided in the literature what the appropriate requirement should be. 
One idea, called representation equivalence, takes the set of discretizations $\mathcal{D}$ to be frame-theoretic analysis operators and requires exact equality \citep{bartolucci_2023}.
Another idea, called discretization invariance, takes the set of discretizations $\mathcal{D}$ to be sampling operators on different meshes and requires approximate equality with convergence as the mesh becomes finer \citep{kovachki_2024}. 
Both ideas fit into our framework of numerical operator learning and can be viewed as special cases of discretization independence.
For example, discretization independence in operator learning reduces to representation equivalence if $\inputspace$ and $\outputspace$ are chosen to be Hilbert spaces, $\mathcal{D}$ consists of frame-theoretic analysis operators, and the approximations in \eqref{eq:numerical_operator_learning} are taken to be exact equalities.
Similarly, discretization independence reduces to discretization invariance if $\inputspace$ and $\outputspace$ are taken to be spaces of continuous functions on bounded domains $\Omega_\inputspace \subseteq \R^{d_\inputspace}$ and $\Omega_\outputspace \subseteq \R^{d_\outputspace}$, the set $\mathcal{D}$ consists of sampling discretizations, and the approximation in \eqref{eq:numerical_operator_learning} is taken to mean convergence uniformly in $\outputspace$ as the maximum distance between pairs of distinct sampling points approaches zero.

\begin{remark}
\label{rem:use_function_spaces}
Since nearly all operator learning models are formulated to learn mappings between function spaces, we will assume from now on that $\inputspace \subseteq L^2(\inputdomain{})$ and $\outputspace \subseteq L^2(\outputdomain{})$, where $\inputdomain{} \subseteq \R^{\inputdimension}$ and $\outputdomain{}\subseteq \R^{\outputdimension}$ are bounded domains. In other words, $\inputspace$ and $\outputspace$ are function spaces. We will add further assumptions about $\inputspace$ and $\outputspace$ as necessary.  

Some problems consider operators on product spaces whose factors are function spaces, that is, operators with multiple inputs or outputs; however, handling multiple inputs and outputs is not the purpose of this work, so we focus on learning operators with a single input function and a single output function.
\end{remark}

\noindent \textbf{{\cn Multiresolution} learning } 
{\cn Multiresolution data 
arises naturally when using numerical simulations because most numerical solvers have user-specified parameters that control the discretization.
The main benefit of multiresolution data is that it can be used to optimize the cost-accuracy tradeoff during training: whereas low-resolution data is usually cheap to use in training but reduces accuracy, high-resolution data is usually expensive to use but increases accuracy \citep{subedi_2026,sakarvadia_2026}. 
Using only low- or high-resolution data may result in inadequate performance or unacceptable cost, but using the right mixture of resolutions results in adequate performance at the lowest possible cost.}

{\cn A key challenge in multiresolution learning is to develop a method that can fuse different representations of the same underlying system. 
Numerical operator learning is well-suited to handling this challenge because the resolution of the data corresponds precisely to the discretization of the data.
Indeed, suppose that $u \in \inputspace$ is the true data, and let ${\inputdiscretization^\ell(u) = \vec{u}_\ell \in \R^{m_\ell}}$ and  $\inputdiscretization^h(u) = \vec{u}_h \in \R^{m_h}$ denote its low- and high-resolution representations, respectively.}
A numerical operator learning model $\numericalmodel{\modelparam}{\inputdiscretization}{\outputdiscretization}{\cdot}$ automatically handles these different representations by \eqref{eq:numerical_operator_learning}:
\begin{equation*}
    \numericalmodel{\modelparam}{\inputdiscretization^\ell}{\outputdiscretization}{\vec{u}_\ell} \approx \outputdiscretization(\model{\modelparam}(u)) \approx \numericalmodel{\modelparam}{\inputdiscretization^h}{\outputdiscretization}{\vec{u}_h},
\end{equation*}
where $\outputdiscretization$ is any output discretization. 

{\cn In Definition \ref{def:numerical_operator_learning} the differentiating aspect of numerical operator learning is the explicit handling of different discretizations.
Operator learning handles multiresolution data gracefully because of the explicit treatment of different discretizations; hence, applying operator learning to multiresolution problems highlights its most important strength.}
Therefore, we argue that a good operator learning model should emphasize this difference; although, we note that using the framework of operator learning also provides a conceptual advantage by suggesting the context in which a model should be applied.

\subsection{Brief review of existing methods} 
\label{sec:background:review}

Below we briefly review existing operator learning approaches, focusing on their discretization independence and other key properties. 
These methods can be broadly classified into three groups according to their  overall structure: neural operators \citep{li_2020, li_2021a}, transformers \citep{cao_2021, li_2023, hao_2023}, and encode-approximate-reconstruct models \citep{kovachki_2024a,hao_2023}. 

Neural operators \citep{kovachki_2024, li_2020, li_2021a, li_2023b} transform an input function through a sequence of functional hidden states by alternating integral transforms and pointwise nonlinear operators. 
They are formulated at the theoretical level of operator learning but implemented at the practical level of numerical operator learning. 
This makes the model more interpretable and easier to analyze \citep{kovachki_2021}, but more importantly it means that discretization independence can be achieved by carefully discretizing each component of the model.  
As a result, neural operators can be applied to any discretizations for which the integral transforms and pointwise nonlinear operators can be approximated. 
The discretization of neural operators comes with the major challenge of efficiently implementing the general integral transform operator.
The most successful neural operator, FNO \citep{li_2021a}, makes the assumption that the integral transform is a convolution and relies on FFTs  to approximate the resulting convolution integrals efficiently. 
The FFT algorithm, however, only works when the input function is discretized by sampling on a uniform grid, so FNO can be applied only when the input and output functions are discretized by sampling on a uniform grid. 
So far, there has been some work applying neural operators to {\cn multiresolution data \citep{thakur_2022, sakarvadia_2026, subedi_2026}, demonstrating potential failures to generalize to high-resolution test data when training only with low-resolution data.}

Transformer methods \citep{hao_2023, li_2023, li_2023a, cao_2021,  fonseca_2023, ovadia_2024, alkin_2024} assume the use of a sampling mesh and treat the function samples on the mesh as tokens. 
Since transformers make use of the attention mechanism from natural language processing, which can handle inputs with different sizes, it follows that transformers can be readily applied to arbitrary sampling meshes, giving them similar discretization independence to neural operators (before additional assumptions that limit discretization independence are made, as in FNO). 
Transformers offer a very general framework for operator learning due to the flexibility of the attention mechanism; for example, the General Neural Operator Transformer (GNOT) \citep{hao_2023} uses attention to handle heterogeneous inputs, such as PDE boundary conditions, vector parameters, and functional parameters.
Transformers, however, are formulated directly from the level of numerical operator learning, so most transformers lack a theoretical interpretation. 
One exception is with Fourier/Galerkin transformers, in which attention is interpreted as a discretization of integral transforms \citep{cao_2021}, which makes it possible to interpret transformers as a kind of numerical realization of neural operators \citep{kovachki_2024}. 
To the best of our knowledge, {\cn there is no work to date investigating the use of transformers for multiresolution operator learning.} 

Encode-approximate-reconstruct models are formulated at the level of operator learning in the abstract framework of Lanthaler et al.~\citep{lanthaler_2022}, but, so far, most existing models using this approach have limited discretization independence, with some being applicable only to specific input and output discretizations \citep{bhattacharya_2021, hesthaven_2018, lu_2021}, for example, by applying a neural network directly to function samples \citep{lu_2021} or by applying proper orthogonal decomposition for input or output functions at a fixed resolution \citep{hesthaven_2018, bhattacharya_2021, lu_2022}.
In recent works, several strategies have been proposed to improve discretization independence.
One encode-approximate-reconstruct model  uses an attention layer in the encoder to improve discretization independence \citep{prasthofer_2022}.
Another approach improves discretization independence by incorporating  a graph neural network \citep{sun_2023a}.
These models, like transformers, are difficult to interpret as operator learning models and lose the theoretical approximation guarantees that earlier models possess.
One recent work uses neural network basis functions to improve discretization independence but uses a two-stage training process as opposed to our end-to-end approach \citep{ingebrand_2025}. 

Compared to neural operator and transformer methods, encode-approximate-reconstruct models are easier to interpret, with a well-developed and general approximation theory  \citep{lanthaler_2022, schwab_2023}. Nevertheless, the most popular models fail to achieve discretization independence comparable to neural operators and transformers. One of our main motivations, therefore, is to fill this gap by creating a simple, general encode-approximate-reconstruct model within the framework of operator learning that retains theoretical guarantees and still has as much flexibility as the neural operator and transformer methods in the types of discretizations it can handle. Additionally, {\cn we will demonstrate how our model can be trained with multiresolution data, showcasing how multiresolution training can improve accuracy while decreasing cost.}

\section{Proposed operator learning model}
\label{section:model}
\setcounter{equation}{0}

In this section we introduce a novel operator learning model within the encode-approximate-reconstruct framework, present a detailed error analysis of our approach, and highlight its unique properties that distinguish it from existing methods in the same framework. 
Our model overcomes the limitations of discretization dependence of existing methods \citep{bhattacharya_2021, hesthaven_2018, lu_2021}, and thus provides a more general framework for encode-approximate-reconstruct approaches. 
Consequently, our model can effectively handle {\cn multiresolution} data so as to benefit the training and performance of the learned model. 
Moreover, inherited from the encode-approximate-reconstruct approach, our model can easily handle unstructured data, in contrast to other models, such as FNO \citep{li_2021a}.  

We will start by introducing the encode-approximate-reconstruct framework. 
It provides a general approach for operator approximation by first encoding the input to a finite-dimensional vector of fixed dimension, then approximating a fixed finite-dimensional representation of the output function, and finally reconstructing the output function. 
More precisely, the model $\model{\modelparam} : \inputspace \to \outputspace$ is decomposed into three components, each of which is a parametric model: an encoder ${\encoder(\encoderparams; \cdot) : \inputspace \to \R^{\inputlatentdim}}$, an approximator $\approximator(\approxparams; \cdot) : \R^{\inputlatentdim} \to \R^{\outputlatentdim}$, and a reconstructor $\reconstructor(\reconstructparams; \cdot): \R^{\outputlatentdim} \to \outputspace$; the dimensions $p$ and $q$ are hyperparameters.
Thus, the model $\model{\modelparam}$ is given by 
\begin{equation}
    \model{\modelparam} = \reconstructor(\reconstructparams;\cdot) \circ \approximator(\approxparams;\cdot) \circ \encoder(\encoderparams; \cdot),
\end{equation}
where the parameter vector $\modelparam = (\encoderparams, \approxparams, \reconstructparams)$.
It provides a general framework with the theoretical advantages that the total model error can be analyzed by examining the properties of each component $\encoder$, $\approximator$, and $\reconstructor$ in isolation. 

Following \citep{lanthaler_2022}, we assume that both the encoder and the reconstructor are linear mappings. 
By the Riesz representation theorem, linear mappings between the Hilbert space $L^2$ and a finite-dimensional space can be uniquely determined by a finite set of basis functions in $L^2$. 
Therefore, we propose to approximate the most general class of linear encoders and reconstructors by using neural networks to represent their associated basis functions.  
Next, we describe the encoder, decoder, and approximator of our model, and then we discretize the resulting operator learning framework into a numerical model, following a paradigm similar to that of neural operators \citep{kovachki_2024}. 

\vspace{0.5em}
\noindent\textbf{Linear encoder }
As assumed in Remark \ref{rem:use_function_spaces},
let $\inputspace$ be a subspace of $L^2(\inputdomain{})$. 
Note that an arbitrary linear mapping $\psi: \inputspace \to \R^{\inputlatentdim}$ is given by $\psi = (\psi_1, \ldots, \psi_p)^T$, where $\psi_j : \inputspace \to \R$ is a linear functional on $\inputspace$ for $j=1,2,\ldots,p$.
By the Riesz representation theorem, 
for each bounded linear functional $\psi_j$ there exists $f_j \in L^2(\inputdomain)$ such that $\psi_j(u) = \langle u, f_j \rangle$ for all $u \in \inputspace$. 
We propose to approximate the functions $\{f_j\}_{j=1}^p$ by a neural network $\encodernet : \R^{\inputdimension} \to \R^{\inputlatentdim}$. 
Thus, we formulate our linear encoder model $\encoder(\encoderparams; \cdot) : \inputspace \to \R^{\inputlatentdim}$ as:
\begin{equation}\label{Int-encoder}
    \encoder_j(\encoderparams; u) = \langle u, \encodernet_j\rangle = \int_{\inputdomain} u\encodernet_j, \qquad j = 1,\ldots, \,\inputlatentdim,
\end{equation}
where $\encoder_j$ and $\encodernet_j$ are the $j$th components of the vector-valued functions $\encoder$ and $\encodernet$, respectively. 
The parameters $\encoderparams$ of the encoder are those of the neural network $\encodernet$.
Note that the only part that requires discretization is the integral in (\ref{Int-encoder}), thus placing the responsibility for discretization independence on the numerical integration, as is done in neural operators \citep{kovachki_2024}. 

\vspace{0.5em}
\noindent\textbf{Linear reconstructor }  Note that any  linear mapping $\varphi : \R^{\outputlatentdim} \to \outputspace$ can be represented by using the input as the coefficients in a linear combination of elements of $\outputspace$. That is, there exist  $g_1, \ldots, g_{\outputlatentdim} \in \outputspace$ such that $\varphi(\vec{w}) = \sum_{j=1}^{\outputlatentdim} w_jg_j,$
where $w_j$ is the $j$th component of $\vec{w} \in \R^{\outputlatentdim}$. 
Similar to the encoder basis, we use a neural network \( \reconstructnet : \R^{\outputdimension} \to \R^{\outputlatentdim} \) to approximate the reconstruction basis $\{g_j\}_{j=1}^q$. 
Then we can formulate our reconstructor $\reconstructor$ as
\begin{equation}
    \reconstructor(\reconstructparams; \vec{w}) = \sum_{j=1}^{\outputlatentdim} w_j\reconstructnet_j, \quad 
\end{equation}
where $\reconstructnet_j$ is the $j$th component  of the vector-valued function $\reconstructnet$, and $\reconstructparams$ are the parameters of  $\reconstructnet$. 
Note that discretization is essentially not an issue for the reconstructor, since the formula for $\reconstructor$ explicitly defines a function in $\outputspace$ using the finite set of parameters from the neural network \( \reconstructnet \). As a result, the output of \( \reconstructor \) can, in principle, be discretized in any desired manner, although some discretization strategies require more work than others. 
For example, evaluating the output function on a sampling mesh is fairly straightforward, whereas projection onto a basis, though not conceptually difficult, requires additional computation. 

\vspace{0.5em}
\noindent\textbf{Approximator } For the approximator, we only need to model a mapping from $\R^{\inputlatentdim}$ to $\R^{\outputlatentdim}$. 
Since no discretization is required, we simply define the approximator as $\approximator = \approxnet$, where $\approxnet$ is another neural network.

\vspace{0.3em}
\noindent Combining these three components, we obtain our parametric operator
\begin{equation}\label{eq: Gours}
    \ourmodel = (\reconstructor\circ \approximator\circ \encoder): \inputspace \to \outputspace
\end{equation}
where our encoder $\encoder$, approximator $\approximator$, and reconstructor $\reconstructor$ are defined as
\begin{equation}
    \label{eq:our_model}
    \begin{cases}
        \encoder(\encoderparams; u) = \big(\langle u, \encodernet_1\rangle,  \ldots,  \langle u, \encodernet_{\inputlatentdim}\rangle\big) \in \R^{\inputlatentdim}, &\qquad u\in \inputspace,\\[0.3em]
        \approximator(\approxparams; \vec{z}) = \approxnet(\vec{z}) \in \R^{\outputlatentdim}, &\qquad\vec{z}\in \R^{\inputlatentdim},\\[0.3em]
        \reconstructor(\reconstructparams; \vec{w}) = \sum\limits_{j=1}^{\outputlatentdim}w_j \reconstructnet_j \, \in \outputspace, &\qquad\vec{w}\in\R^{\outputlatentdim},
    \end{cases}
\end{equation}
where $\encodernet = (\encodernet_1, \,\ldots, \, \encodernet_{\inputlatentdim})$, $\reconstructnet = (\reconstructnet_1, \,\ldots,\, \reconstructnet_{\outputlatentdim})$, and $\vec{w} = (w_1, \,\ldots,\, w_{\outputlatentdim})$.

\subsection{Discretization independence and comparison}
\label{sec:methods:discretization-and-comparison}

As previously discussed, the discretizations $\inputdiscretization$ and $\outputdiscretization$ play an important role in numerical operator learning models.  
Hence, before we can define a numerical approximation of our model,  we need to determine the set $\mathcal{D}$ of discretizations to which it will be applied.  
In our model (\ref{eq:our_model}), the reconstructor can essentially be used with any output discretization, since the output of our model explicitly defines a function using a finite number of parameters $\theta_R$. 
However, the input relies on the computation of the inner product 
\begin{equation}\label{inner-int}
\langle u, \encodernet_j\rangle = \int_{\inputdomain{}}u\encodernet_j, \qquad 1 \le j \le p.
\end{equation}
Given that $\encodernet_j$ is an explicitly known function (a neural network), the approximation of the inner product in (\ref{inner-int}) depends only on whether it can be accurately evaluated from the discrete representation $\inputdiscretization(u)$. 

Let $\mathcal{D}_I$ be the set of discretizations $\inputdiscretization:  \inputspace \to \R^\inputdiscdimension$ for which a numerical quadrature rule $I(\phi, \cdot)$ exists to approximate $\langle u, \phi\rangle$ for a given function $\phi$, that is,
\begin{equation*}
    \langle u, \phi\rangle \approx  I(\phi, \inputdiscretization(u)), \qquad \text{for all } u \in \inputspace.
\end{equation*}
Then our model can be applied numerically for the set of discretizations
\begin{equation}
    \mathcal{D} = \{(\inputdiscretization, \outputdiscretization) : \inputdiscretization \in \mathcal{D}_I\}
\end{equation}
and our numerical model $\widetilde{G}_\theta$ is given by
\begin{equation}
    \numericalmodel{\modelparam}{\inputdiscretization}{\outputdiscretization}{\inputdiscretization(u)} = \reconstructor\circ \approximator \circ \widetilde{\encoder}(\inputdiscretization;\cdot),
\end{equation}
where the numerical encoder $\widetilde{\encoder}(\inputdiscretization;\cdot)$ is given by
\begin{equation}
    \widetilde{\encoder}(\inputdiscretization; \vec{u}) =  (I(\encodernet_1, \vec{u}), I(\encodernet_2, \vec{u}), \ldots, I(\encodernet_{\inputlatentdim}, \vec{u})), \qquad \vec{u} \in \R^\inputdiscdimension.
\end{equation}
Thus, our method can effectively handle any input discretization for which numerical integration is feasible and any choice of output discretization. The responsibility of maintaining discretization independence in our model is placed on the numerical integration, which is essentially the same strategy used by neural operators \citep{kovachki_2024}. 
The accuracy of the numerical integration ${I}(\Phi_j^E, {\bf u})$ depends on the resolution of ${\bf u}$ and the choice of quadrature rule. Our extensive studies suggest that using a higher-order quadrature rule could enhance discretization independence by increasing the accuracy of numerical integration for smooth, low-{\cn resolution} input functions. This would also allow our model to use lower-{\cn resolution} training data, improving {\cn training} efficiency. 
Similar to other methods, our model breaks down when the input {\cn resolution} becomes too low for numerical integration to accurately approximate the inner product, due to insufficient information in the input data (see, in particular, Figure \ref{fig:navier_stokes_plots_input_low}).

{\cn We remark that the number of basis functions used in the encoder ($\inputlatentdim$) and reconstructor ($\outputlatentdim$) controls the expressivity of our model, in a manner similar to the role of retained modes in FNO. 
In our framework, increasing the number of basis functions enlarges the approximation space and can improve representation capacity. 
However, this choice is conceptually distinct from discretization independence: 
whatever operator $\model{\modelparam}$ our model realizes, it is approximated numerically by $\numericalmodelnaked{\modelparam}$ in a discretization-independent manner, regardless of how well $\model{\modelparam}$ approximates the target operator $\trueop$.}

\smallskip
\newcommand{\discdon}{\mathcal{D}^{\text{DeepONet}}}
\newcommand{\discpoddon}{\mathcal{D}^{\text{POD-DeepONet}}}
\newcommand{\discpcanet}{\mathcal{D}^{\text{PCA-Net}}}
Next, we compare discretization independence of our method and other encode-approximate-reconstruct methods by characterizing their effective discretization sets $\mathcal{D}$.
Our method can effectively handle arbitrary input and output discretizations, which significantly distinguishes it from DeepONet \citep{lu_2021}, PCA-Net \citep{bhattacharya_2021} (also called POD-NN by Hesthaven and Ubbiali \citep{hesthaven_2018}), and POD-DeepONet \citep{lu_2022}. 
In contrast, DeepONet is limited to a single input discretization, namely, sampling at a set of predefined sensor points \citep{lu_2021}.
The same restriction applies to PCA-Net and POD-DeepONet, which additionally require the output discretization to lie on a fixed sampling grid, unless interpolation is employed \citep{bhattacharya_2021}. 

To facilitate more detailed comparison, we formulate all models in a unified encode–\\approximate–reconstruct framework.\footnote{This is not how these (i.e., DeepONet, PCA-Net, POD-NN, and POD-DeepONet) methods were originally formulated, but presenting them in this way makes it easier to compare with ours.}
All methods being compared rely on sampling discretization in some way; therefore, we develop the following notation to describe their effective discretization sets.
Let $\mathcal{D}_S$ and $\mathcal{D}_Q$ denote sets of input and output discretizations that can provide sample function values on the point sets $S = \{\vec{x}_j\}_{j=1}^m \subseteq \inputdomain$ and ${Q = \{\vec{y}_j\}_{j=1}^n \subseteq \outputdomain}$. 
That is, $\inputdiscretization \in \mathcal{D}_S$ if and only if there exists a mapping $\pi_S$ such that ${\pi_S(\inputdiscretization(u)) = (u(\vec{x}_1), \dots, u(\vec{x}_m))^T}$ for all $u \in \inputspace$; membership in $\mathcal{D}_Q$ requires the existence of an analogous mapping $\pi_Q$. 
The main differences between methods lie in the numerical treatment of the encoder and reconstructor; see Table \ref{table:methods:comparison}. 

\begin{table}[htb!]
\setlength\extrarowheight{6pt}
\centering
\begin{tabular}{@{}llll@{}}
\toprule
Method & Numerical model $\widetilde{G}_\modelparam$ & 
$\widetilde{\encoder}(\inputdiscretization; \vec{u})$ and $\widetilde{\reconstructor}(\outputdiscretization; \vec{w})$
& DI condition \\
\midrule
\textbf{Ours} & \multirow[t]{2}{*}{$\displaystyle \reconstructor\circ \approximator \circ \widetilde{\encoder}(\inputdiscretization;\cdot)$} & $\displaystyle \widetilde{\encoder} = (I(\encodernet_1, \vec{u}),  \dots, I(\encodernet_{\inputlatentdim}, \vec{u}))$ & \textbf{None} \\
DeepONet &  & $\widetilde{\encoder}= (u(\vec{x}_1), \dots, u(\vec{x}_{p}))$ & $\inputdiscretization \in \mathcal{D}_{S}$ \\
\multirow[t]{2}{*}{POD-DeepONet} & \multirow[t]{4}{*}{$\displaystyle\widetilde{\reconstructor}(\outputdiscretization; \cdot) \circ \approximator\circ \widetilde{\encoder}(\inputdiscretization; \cdot)$} & $\displaystyle\widetilde{\encoder} = (u(\vec{x}_1), \dots, u(\vec{x}_{p}))$ & $\inputdiscretization \in \mathcal{D}_{S}$ \\
&  & $\widetilde{\reconstructor} = \pi_Q\Big(\vec{f}_0 + \sum_{j=1}^{\outputlatentdim} w_j \vec{f}_j\Big)$ & $\outputdiscretization \in \mathcal{D}_{Q}$\\
\multirow[t]{2}{*}{PCA-Net} & &$\widetilde{\encoder} = (\vec{g}_1^T\pi_S(\vec{u}), \dots ,\vec{g}_{\inputlatentdim}^T\pi_S(\vec{u}))$ & $\inputdiscretization \in \mathcal{D}_{S}$ \\
& & $\widetilde{\reconstructor} = \pi_Q\left(\sum_{j=1}^{\outputlatentdim} w_j \vec{f}_j\right)$ & $\outputdiscretization \in \mathcal{D}_{Q}$\\
\bottomrule
\end{tabular}
\vskip -1mm\caption{Comparison of our method and DeepONet, PCA-Net/POD-NN and POD-DeepONet, for their numerical models and  discretization independence (DI) condition. For POD/PCA methods, the values $\{w_j\}_{j=1}^q$ are the coefficients predicted by the approximator.}\label{table:methods:comparison}
\vspace{-5mm}
\end{table}
The encoder of DeepONet is given by sampling on the set $S$.  
Then the numerical encoder $\widetilde{\encoder}$ of DeepONet is simply given by $\pi_S$.
The POD-DeepONet \citep{lu_2022} model uses the same numerical encoder as DeepONet, but further approximates the reconstructor using a set of basis functions $\{g_\ell\}_{\ell=1}^q$ obtained via proper orthogonal decomposition (POD) applied to the output functions {$\{\dataelem{v}{i}\}_{i=1}^N$} in the training data. 
It assumes that the discretized function data provide the sample values on the points in $Q$; that is, $\outputdiscretization^i \in \mathcal{D}_Q$ for $i=1,2,\dots, N$.
Numerically, POD is implemented by performing principal component analysis on the samples of the input functions on $Q$, yielding a set of vectors $\{\vec{g}_\ell\}_{\ell=1}^q$ that approximate the sample values on $Q$ of the true POD basis functions $\{g_\ell\}_{\ell=1}^q$.
In other words, POD-DeepONet can evaluate the output function only at the points $Q$, as the POD basis functions are known only at those points. 
Consequently, in addition to assuming that functions in $\inputspace$ are defined at the points $S$, one must also assume that functions in $\outputspace$ are well-defined at the points $Q$.  

PCA-Net \citep{hesthaven_2018, bhattacharya_2021} uses POD for both the encoder and reconstructor, where $S$ and $Q$ are used as the sampling points for the numerical implementation of POD for the input and output functions.
Thus, one obtains the vectors $\{\vec{f}_k\}_{k=1}^p$, which give the approximate sample values of $\{f_k\}_{k=1}^p$, the POD basis functions for the input.
This makes PCA-Net essentially as discretization dependent as POD-DeepONet. 
We note that the independence in POD-DeepONet and PCA-Net is only approximate, as the vectors $\vec{g}_\ell$ and $\vec{f}_k$ only approximate the functions $g_\ell$ and $f_k$.

\begin{remark}
    DeepONet cannot approximate the optimal linear encoder of fixed dimension $\inputlatentdim$ to an arbitrary degree of accuracy. It can, in a probabilistic sense with randomly chosen sensors, achieve optimal accuracy asymptotically (see Theorem 3.7 in \citep{lanthaler_2022}), but it cannot approach the optimal encoding error for a linear encoder with any fixed dimension $\inputlatentdim$.  
    In PCA-Net the encoder is already given by projection onto the optimal subspace of dimension $\inputlatentdim$ (or an approximation thereof, at least) due to the optimality of POD; however, the method has very strict discretization dependence. Our model maintains the optimality properties of PCA-Net and at the same time is discretization independent. 
\end{remark}
\begin{remark}
    Considering sampling discretizations adds the requirement that the input and output functions be defined at the sampling points, which is a strong assumption for functions in $\inputspace \subseteq L^2(\inputdomain{})$, as functions in $L^2$ are not necessarily defined pointwise.
\end{remark}
Finally, we remark that requiring the input and/or output functions to be known at specific sets of points significantly limits the flexibility of DeepONet, POD-DeepONet, and PCA-Net. 
In particular, this constraint often makes {\cn multiresolution} learning infeasible for these methods unless substantial modifications are made, such as using multiple model instances to handle different {\cn fidelity levels (corresponding to different resolutions in our setting)}, as proposed in \citep{lu_2022a, howard_2023, de_2023}.

\subsection{Approximation theory}

In this section, we present two theorems that establish the universal approximation capabilities of our method, building on results from the encode-approximate-reconstruct framework  \citep{lanthaler_2022, herrmann_2024, schwab_2023}. The first theorem shows that our method can uniformly approximate continuous operators on compact domains. The second theorem demonstrates that our method can approximate Lipschitz continuous operators in a statistical sense, for any distribution over the input functions.

The first theorem builds on the theoretical operator learning architecture proposed by Schwab et al.~\citep{schwab_2023}, which we restate below in our notations. 
\begin{theorem}{\textnormal{(Schwab, Stein and Zech, 2023)}}
   \label{thm:shallowonet}
    Let $\inputspace_0$ and $\outputspace$ be separable Hilbert spaces, and let $\inputspace \subseteq \inputspace_0$ be a compact set. If $\trueop : \inputspace \to \outputspace$ is continuous, then for any $\varepsilon > 0$, there exists an encode-approximate reconstruct model $\model{\modelparam}$ such that
   \begin{equation}
       \label{eq:zech-schwab-stein}
       \forall u \in \inputspace \qquad \|\model{\modelparam}(u) - \trueop(u)\|_{\outputspace} < \varepsilon.
   \end{equation}
\end{theorem}
In the remainder of this section, we let
\bea\label{eq:Gtheta}
G_\theta = \big(R \circ A \circ E\big)
: \inputspace \to \outputspace
\eea
denote the model whose existence is guaranteed in Theorem~\ref{thm:shallowonet}. 
In the setting of Theorem~\ref{thm:shallowonet}, the encoder $E : \inputspace_0 \to \R^{\inputlatentdim}$ is defined by 
$E(u)=  (\langle u, \psi_1\rangle, \dots, \langle u, \psi_p \rangle)$, for $\psi_1, \dots, \psi_p \in \inputspace_0$. The approximator $A : \R^{\inputlatentdim} \to \R^{\outputlatentdim}$
is a continuous neural network with activation
$\sigma \in L^\infty_{\mathrm{loc}}(\mathbb{R})$ that is non-polynomial and
continuous almost everywhere. The reconstructor
$R : \R^{\outputlatentdim} \to \outputspace$ is given by
$R(\vec{w}) = \sum_{i=1}^q w_i \varphi_i$, where $\vec{w} = (w_1, \dots, w_q)$ and $\varphi_1, \dots, \varphi_q \in \outputspace$. 

Next, we present a uniform approximation theorem for our method. 
In the following proof the model $G_{\mathrm{ours}}$ refers to our architecture in \eqref{eq: Gours},
with $E(\theta_E;\cdot)$, $A(\theta_A;\cdot)$, and $R(\theta_R;\cdot)$
defined in \eqref{eq:our_model}. We write $E(\theta_E;\cdot)$, $A(\theta_A;\cdot)$, and $R(\theta_R;\cdot)$ as $E_{\mathrm{ours}}$, $A_{\mathrm{ours}}$, and $R_{\mathrm{ours}}$ to distinguish from those in \eqref{eq:Gtheta}. 

\begin{theorem}
    \label{thm:shallowonet_ours}
    Suppose that $\inputdomain$ and $\outputdomain$ are compact. Let $\inputspace \subseteq L^2(\inputdomain)$ be compact, and let $\outputspace = L^2(\outputdomain)$. 
    Assume that
    $\trueop : \inputspace \to \outputspace$ is continuous.
Then for any $\varepsilon > 0$, there exists a model $\ourmodel$ of the form in \eqref{eq: Gours}--\eqref{eq:our_model} such that
    \begin{equation}
        \label{eq:uniform_approximation}
        \forall u \in \inputspace \qquad \|\ourmodel(u) - \trueop(u)\|_{L^2(\outputdomain)} < \varepsilon.
    \end{equation}
\end{theorem}

\begin{proof}
   The proof can be done by applying Theorem \ref{thm:shallowonet}. 
   Suppose that $\model{\modelparam}$ is the encode-approximate-reconstruct model in \eqref{eq:Gtheta} whose existence is guaranteed by Theorem \ref{thm:shallowonet}. Define $Z = E(K) \subseteq \R^\inputlatentdim$ to be the image of $K$ under the encoder $E$. Since the inner products in the definition of $E$ are continuous mappings, it follows that $E$ is continuous. Hence, $Z$
is compact, being the continuous image of a compact set.
 Define $W = A(Z) \subseteq \R^\outputlatentdim$ to be the image of $Z$ under the approximator $A$. Since $A$ is continuous,  $W$ is also compact. 
 We now construct our model to approximate each component of $\model{\modelparam}$.
 
\smallskip
\noindent{\bf Step 1 (Encoder).}
     We begin by showing that for any $\varepsilon > 0$, there exists an encoder $E_\text{ours}$ such that
   \begin{equation}\label{eq:Eterm}
       \forall u \in \inputspace \qquad \|E_\mathrm{ours}(u) - E(u)\|_\infty < \varepsilon.
   \end{equation}
By the universal approximation property of neural networks
(e.g., Theorem~3.1 in \citep{pinkus_1999}), there exists a neural network
$\encodernet : \R^{\inputdimension} \to \R^{\inputlatentdim}$ such that
$\|\encodernet_i - \psi_i\|_{L^2(\inputdomain)} < \varepsilon/C_1$
for $i = 1, \dots, p$, where
$C_1 = \sup\limits_{u \in \inputspace} \|u\|_{L^2(\inputdomain)} < \infty$
since $\inputspace$ is compact. Then, by the Cauchy-Schwarz inequality,
   \begin{align*}
       \|E_\mathrm{ours}(u) - E(u)\|_\infty &= \max_{i=1,\,\dots,\, p} \big|\langle u, \encodernet_i\rangle - \langle u, \psi_i\rangle\big| \\
       &\le \max_{i=1,\,\dots,\,p} \|u\|_{L^2(\inputdomain)}\|\encodernet_i - \psi_i\|_{L^2(\inputdomain)} \le \varepsilon.
   \end{align*}
   
\noindent{\bf Step 2 (Approximator).}
     Next, we show that for any $\varepsilon > 0$,  there exists an approximator $A_\mathrm{ours}$ and an encoder $E_\mathrm{ours}$ such that 
   \begin{equation}\label{eq:Aterm}
       \forall u \in \inputspace \qquad \| A_\mathrm{ours}(E_\mathrm{ours}(u)) - A(E(u))\|_\infty < \varepsilon.
   \end{equation}
   Define $Z' = \{\vec{z} \in \R^\inputlatentdim : d(\vec{z}, Z) \le 1\}$, where $d(\cdot,Z)$ is the $\ell^\infty$ distance to the set $Z \subseteq \R^\inputlatentdim$. Then $Z'$ is closed and bounded, and thus compact in $\R^\inputlatentdim$; additionally, we have $Z \subseteq Z'$. Since $A$ is  continuous $\R^\inputlatentdim$, it is uniformly continuous on $Z'$. Therefore, there exists $\delta < 1$ such that $\|\vec{z}_1 -\vec{z}_2\| < \delta$ implies that $\|A(\vec{z}_1) - A(\vec{z}_2)\|_\infty < \frac{\varepsilon}{2}$ for any $\vec{z}_1, \vec{z}_2 \in Z'$.

   By the result of Step 1, we can find an encoder $E_\text{ours}$ such that ${\|E_\text{ours}(u) - E(u)\|_\infty < \delta}$ for all $u \in \inputspace$. Since $\delta < 1$, it follows that $E_\text{ours}(u) \in Z'$, and certainly $E(u) \in Z'$. Hence, we get $\|A(E_\text{ours}(u)) - A(E(u))\|_\infty < \frac{\varepsilon}{2}$ for all $u \in \inputspace$. Furthermore, by the universal approximation property of neural networks, there exists a neural network ${\approxnet : \R^\inputlatentdim \to \R^\outputlatentdim}$ such that $\|\approxnet(\vec{z}) - A(\vec{z})\|_\infty < \frac{\varepsilon}{2}$ for all $\vec{z} \in Z'$. Combining the above results, we get
   \begin{align*}
        \|A_\text{ours}(E_\text{ours}(u)) - A(E(u))\|_\infty & =  \|\approxnet(E_\text{ours}(u)) - A(E(u))\|_\infty\\ 
        &\le \|\approxnet(E_\text{ours}(u)) - A(E_\text{ours}(u))\|_\infty + \|A(E_\text{ours}(u)) - A(E(u))\|_\infty \\
        & \le \frac{\varepsilon}{2} + \frac{\varepsilon}{2} = \varepsilon.
   \end{align*}
   for all $u \in \inputspace$. 

   \paragraph{Step 3 (Reconstructor). } 
    We complete the proof by defining the reconstructor ${R_\text{ours} : \R^\outputlatentdim \to \outputspace}$ of our model $\ourmodel$. 
   By the universal approximation property of neural networks, we can choose a neural network $\reconstructnet : \R^\outputdimension \to \R^\outputlatentdim$ such that
   \begin{equation}\label{eq:phiR00}
       \|\reconstructnet_i - \varphi_i\|_{L^2(\outputdomain)} < \frac{\varepsilon}{C_2 q}, \qquad i=1,\dots, q,
   \end{equation}
   where $C_2 = \sup\limits_{w\in W} \|w\|_\infty < \infty$ because $W$ is compact.  
   Let $\vec{w} = A(E(u))$, and ${\vec{w}_\text{ours} = A_\text{ours}(E_\text{ours}(u))}$.
   By the results in Steps 1 and 2, we can choose an encoder $E_\text{ours}$ and an approximator $A_\text{ours}$ such that
   \begin{equation}\label{eq:w00}
          \|\vec{w}_\text{ours} - \vec{w}\|_\infty = \|A_\text{ours}(E_\text{ours}(u)) - A(E(u))\|_\infty < \frac{\varepsilon}{C_3q},
   \end{equation}
   for any $u \in \inputspace$, where $C_3 = \max\limits_{i=1,\dots,q}\|\reconstructnet_i\|_{L^2(\outputdomain)}$.  
   
   Given $u \in \inputspace$, we use the triangle inequality to write
  \begin{align*}
    &\|\ourmodel(u) - \trueop(u)\|_{L^2(\outputdomain)} = \|R_\text{ours}(\vec{w}_\text{ours}) - \trueop(u)\|_{L^2(\outputdomain)} \le \text{I} + \text{II} + \text{III}.
   \end{align*}
   We estimate the term $\text{I}$ by
   \begin{align*}
       \text{I} &:= \big\|R_\text{ours}(\vec{w}_\text{ours}) - R_\text{ours}(\vec{w})\big\|_{L^2(\outputdomain)} \\
       &= \Big\|\sum_{i=1}^q (w_{\text{ours},i} - w_i)\reconstructnet_i \Big\|_{L^2(\outputdomain)} \le C_3q\|\vec{w}_\text{ours} - \vec{w}\|_\infty < \varepsilon,
   \end{align*}
   by \eqref{eq:w00}, where $\vec{w}_\text{ours} = (w_{\text{ours},1}, \dots, w_{\text{ours},q})$. Then, we estimate the term $\text{II}$ as
   \begin{align*}
        \text{II} &:=\big\|R_\text{ours}(\vec{w}) - R(\vec{w})\big\|_{L^2(\outputdomain)} =  \Big\|\sum_{i=1}^q w_i \big(\reconstructnet_i - \varphi_i\big)\Big\|_{L^2(\outputdomain)} \le C_2q\|\vec{w}\|_\infty < \varepsilon,
   \end{align*}
   by \eqref{eq:phiR00}. 
   Finally, we estimate the term $\text{III}$ as
   \begin{equation*}
       \text{III} := \|R(\vec{w}) - \trueop(u)\|_{L^2(\outputdomain)} < \varepsilon
   \end{equation*}
   by the choice of $E$, $A$, and $R$ in \eqref{eq:Gtheta} satisfying Theorem~\ref{thm:shallowonet}. 
   Combining the above estimates yields 
   \begin{equation*}
        \|\ourmodel(u) - \trueop(u)\|_{L^2(\outputdomain)} < 3\varepsilon.
   \end{equation*}
   for any $u \in \inputspace$. 
   Since $\varepsilon > 0$ was arbitrary, the result \eqref{eq:uniform_approximation} follows.
\end{proof}  
This approximation result extends to specific discretizations $\inputdiscretization$ and $\outputdiscretization$ in a way that depends on how accurately $\inputdiscretization$ and $\outputdiscretization$ discretely approximate $\inputspace$ and $\outputspace$ (and also on how accurate numerical integration using $\inputdiscretization$ can be). This is an immediate consequence of the continuity of the model with respect to $u$, the use of which was already demonstrated in the proof of the theorem.

\newcommand{\staterror}{\mathscr{E}}
\newcommand{\encodingerror}{\mathscr{E}_E}
\newcommand{\approxerror}{\mathscr{E}_A}
\newcommand{\reconstructerror}{\mathscr{E}_R}
Given a random input $u \sim \mu$, the statistical error $\staterror$ of an operator learning model $\model{\modelparam}$ for the Lipschitz continuous operator $\trueop$ is defined as \citep{lanthaler_2022} 
\begin{equation}
    \staterror^2 = \expect_{u\sim \mu}\left[{\big\lVert \trueop(u) - \model{\modelparam}(u)\big\rVert_{L^2(\outputdomain)}^2}\right],
\end{equation}
that is, the root mean squared $L^2$ error of the predicted output. 
This error can be broken into an \emph{encoding error} $\encodingerror$, an \emph{approximation error} $\approxerror$, and a \emph{reconstruction error} $\reconstructerror$, each of which depends only on the corresponding component of the model, that is, the encoder $\encoder$, the approximator $\approximator$, and reconstructor $\reconstructor$. 
In particular, the total error of the model can be bounded by
\begin{equation}
    \staterror \le C_1\encodingerror + C_2\approxerror + \reconstructerror,
\end{equation}
where $C_1$ and $C_2$ are constants that depend on  $\trueop$ \citep{lanthaler_2022}. The approximation error and reconstruction error for our choice of approximator and reconstructor have already been studied. Specifically, when the approximator is a multilayer perceptron, as in our case, standard neural network approximation theorems provide bounds on the approximation error $\approxerror$ \citep{lanthaler_2022}. Additionally, for linear reconstructors, as in our case, there are bounds on the reconstruction error $\reconstructerror$ based on combining the analysis of finite-dimensional approximations of the space $L^2(\outputdomain)$ with neural network approximation theorems \citep{lanthaler_2022}. Hence, to bound the statistical error $\staterror$, we need only consider the encoding error $\encodingerror$, which is defined as follows.

Given a probability measure $\mu$ on $\inputspace= L^2(\inputdomain)$ with finite second moments, a linear encoder $\encoder : \inputspace \to \R^{\inputlatentdim}$, and a linear mapping $D: \R^{\inputlatentdim} \to \inputspace$ (called the \emph{decoder} in the encode-approximate-reconstruct framework), the encoding error $\encodingerror$ is defined by
\begin{equation}
    \label{eq:encoding_error}
    \encodingerror^2 = \expect_{u\sim \mu}\left[\lVert D(\encoder(u)) - u \rVert_{L^2(\inputdomain)}^2\right].
\end{equation}
Note that we take $\inputspace = L^2(\inputdomain)$ in this case instead of $\inputspace$ being some subspace. This makes sense because we can restrict our input function $u$ to a subspace by choosing $\mu$ supported only on that  subspace. We work in the context of the Hilbert space $L^2(\inputdomain)$ because it makes it straightforward to compute error estimates for our inner-product--based encoder. 

As remarked in \citep{lanthaler_2022}, the encoding error can be decomposed into two components:
\newcommand{\aliasingerror}{\mathscr{E}_{\textnormal{aliasing}}}
\begin{equation}
    \label{eq:error_decomposition}
    \encodingerror^2 =  \aliasingerror^2 + \sum_{k > \inputlatentdim}\lambda_k,
\end{equation}
where $\{\lambda_k\}$ are the eigenvalues, in decreasing order, corresponding to the orthonormal eigenfunctions in the spectral decomposition of the uncentered covariance operator of $\mu$. 
The second term in \eqref{eq:error_decomposition} is independent of the encoder and decoder, and it provides a lower bound on the encoding error with a fixed encoding dimension $\inputlatentdim$ (this bound is also specific to $\mu$, that is, the distribution of the input function $u$). 
The first term, $\aliasingerror^2$, depends on the encoder and decoder. 

In the following theorem, we show that our method can nearly minimize $\aliasingerror$ for a fixed $\inputlatentdim$ by choosing the encoder network $\encodernet$ sufficiently large. 

\begin{theorem}
    \label{thm:optimal_aliasing}
    Let $\mu$ be a probability measure on $L^2(\inputdomain)$, representing the law of a random variable $u$. 
    For any $\varepsilon > 0$ and any positive integer $\inputlatentdim$, there exists a model $\ourmodel$ with encoder network $\encodernet : \R^{\inputdimension}\to \R^{\inputlatentdim}$ (and a suitable decoder) such that the aliasing error $\aliasingerror$ for $u$ satisfies
    \begin{equation*}
        \aliasingerror \le \varepsilon.
    \end{equation*}
\end{theorem}
\begin{proof}
    Let $\Gamma$ denote the covariance operator of $u$, and let $\{\phi_k\}$ be the eigenfunctions in the spectral decomposition of $\Gamma$, so that
    \begin{equation}
        \label{eq:covariance_decomposition}
        \Gamma = \sum_{k=1}^\infty \lambda_k (\phi_k \otimes \phi_k),
    \end{equation}
    where the corresponding eigenvalues $\{\lambda_k\}$ are given in decreasing order. 
    Given $\varepsilon > 0$, standard universality theorems (such as Theorem 3.2 in  \citep{berner_2022}) ensure the existence of a neural network $\encodernet$ such that $\lVert \encodernet_k - \phi_k\rVert_{L^2(\inputdomain)} \le \frac{\varepsilon}{\sigma\sqrt{\inputlatentdim}}$ for all $k=1,2,\dots,\inputlatentdim$, where $\sigma^2$ is the uncentered variance of $u$. 
    Choose the linear decoder $D$ as 
    \begin{equation}
        \label{eq:decoder}
        D(\vec{z}) = \sum_{k=1}^{\inputlatentdim}z_k\phi_k, \qquad \vec{z} \in \R^{\inputlatentdim}.
    \end{equation}
    Then the encoding error can be computed as
    \begin{align}
        \encodingerror^2 &= \int_\inputspace \left\lVert D(\encoder(u)) - u \right\rVert_{L^2(\inputdomain)}^2\dee\mu(u) \nonumber 
        = \int_\inputspace \Big\lVert \sum_{k=1}^{\inputlatentdim}\phi_k \langle\encodernet_k, u\rangle - u \Big\rVert_{L^2(\inputdomain)}^2\dee\mu(u).
        \nonumber 
    \end{align}
    Using the representation of $u$ in the orthonormal basis formed by $\{\phi_k\}$, we have
    \begin{equation}
        \encodingerror^2 = \int_\inputspace \Big\lVert \sum_{k=1}^{\inputlatentdim}\phi_k\big(\langle \encodernet_k, u\rangle - \langle \phi_k, u\rangle\big) - \sum_{k>\inputlatentdim}\phi_k\langle\phi_k, u\rangle\Big\rVert_{L^2(\inputdomain)}^2\dee\mu(u).\nonumber
    \end{equation}
    By the orthonormality of $\{\phi_k\}$, we further obtain
    \begin{equation}
        \encodingerror^2 = \sum_{k=1}^{\inputlatentdim}\int_\inputspace \big|\langle \encodernet_k - \phi_k, u\rangle\big|^2\dee\mu(u) + \sum_{k>\inputlatentdim} \int_{\inputspace} \big|\langle\phi_k, u\rangle\big|^2 \dee\mu(u),
    \end{equation}
where the second term is actually just the lower error bound from \eqref{eq:error_decomposition} (see Appendix C.10 of \citep{lanthaler_2022}). Thus,
    \begin{align}
        \label{eq:aliasing_bound}
        \aliasingerror^2 &= \sum_{k=1}^{\inputlatentdim}\int_\inputspace\big|\langle\encodernet_k - \phi_k, u\rangle\big|^2\dee\mu(u)\nn\\
        &= \sum_{k=1}^{\inputlatentdim}\int_\inputspace \Big|\int_{\inputdomain} \left(\encodernet_k(\vec{x}) - \phi_k(\vec{x})\right) u(\vec{x}) \dee \vec{x}\Big|^2\dee\mu(u) \nn\\
        &\le \sum_{k=1}^{\inputlatentdim}\int_\inputspace \lVert u \rVert_{L^2(\inputdomain)}^2 \left\lVert \encodernet_k - \phi_k \right\rVert_{L^2(\inputdomain)}^2 \dee\mu(u) \le \varepsilon^2.
    \end{align}
    \vspace{-10mm}
\end{proof}
\begin{remark}
     Note that, unlike our method, DeepONet can only reduce $\aliasingerror$ asymptotically as $\inputlatentdim \to \infty$ \citep{lanthaler_2022}.
\end{remark}
\begin{remark}
    Since the approximation error $\approxerror$ and reconstruction error $\reconstructerror$ depend only on the approximator $\approximator$ and reconstructor $\reconstructor$, and these errors can be made arbitrarily small for suitable choices of $\approximator$ and $\reconstructor$ \citep{lanthaler_2022}, it follows from Theorem \ref{thm:optimal_aliasing} that given $\varepsilon > 0$, there exists a model $\ourmodel$ for which the statistical error $\staterror$ satisfies $\staterror < \varepsilon$.
\end{remark}

\section{Numerical experiments}  
\label{section:experiments}

We now carry out numerical experiments to demonstrate the effectiveness of our method and to explore the benefits and limits of {\cn multiresolution} learning. 
We focus on operators arising from diverse PDEs, including one- and two-dimensional, linear and nonlinear, local and nonlocal, and time-dependent and time-independent equations. 
Numerical experiments show that our method is empirically discretization-independent. 
Training with low-{\cn resolution} data {\cn is generally preferred because it is much faster, especially in higher-dimensional problems}; however, lower {\cn resolution} can introduce approximation errors.
We analyze and categorize these errors in detail. 
Furthermore, we show that incorporating a small amount of high-resolution data through {\cn multiresolution} training can effectively mitigate some of these issues. 
This enables the efficient use of low-{\cn resolution} data while maintaining accuracy and reducing computational cost. 

\smallskip
\noindent\textbf{Experimental setup: } 
Consider continuous input and output functions defined on the $d$-dimensional domain $\overline{\Omega}$. 
With the exception of Section \ref{section:Darcy}, we take $\Omega = (0,1)^d$ and consider discretizations obtained by sampling the functions on a uniform grid over $\overline{\Omega}$, including samples at the boundaries. 
As such, we use the trapezoidal rule for numerical integration. 
We note that these decisions about domain and discretization are not essential; they merely simplify the experimental design and implementation. 
Note that any domain, discretization, and quadrature rule can be used with our method as long as numerical integration is possible and sufficiently accurate.
As an example, we show in Section \ref{section:Darcy} how our method handles a problem with a different domain geometry.

To investigate the effects of {\cn multiresolution} learning, we generate many datasets numerically for each example problem, varying the combination of discretizations in each dataset. 
In general,  a dataset can be described by the parameters $N$, $(R_1, \dots, R_M)$, and $(p_1, \dots, p_M)$, where $N$ is the total number of samples in the dataset, $M$ is the number of distinct discretizations. $R_i$ denotes the number of sampling points in one spatial dimension for the $i$th discretization (so that the total number of spatial points is $R_i^d$)\footnote{Again, with the exception of Section \ref{section:Darcy}, where we change the meaning of $R_i$ to conform to the different domain geometry.}, and $p_i$ is the proportion of the samples that use the $i$th discretization.  
Note that the value $R_i$ uniquely determines the discretization, so we will use this value to identify discretizations concisely.
See Figure \ref{fig:example_dataset_diagram} for an illustrative example of a {\cn multiresolution} dataset.
\begin{figure}[htb!]
    \centering
\includegraphics[width=0.7\linewidth]{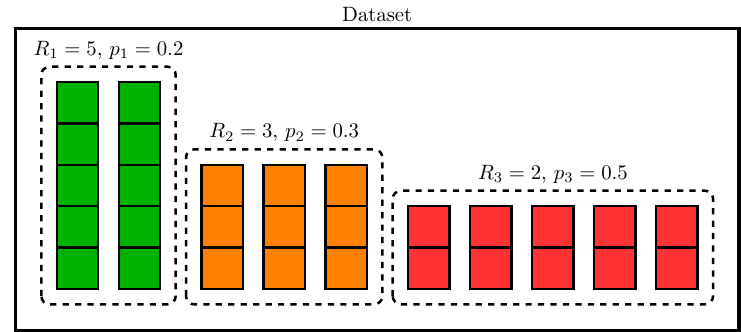}
    \vspace{-3mm}\caption{Example of a {\cn multiresolution} dataset with  $N = 10$ samples and $M = 3$ {\cn resolution} levels.}
    \label{fig:example_dataset_diagram}
    \vspace{-4mm}
\end{figure}
In our setting, we take the same discretization for both the input and output functions and generate {\cn multiresolution} datasets by computing high-resolution solutions numerically and subsampling to lower resolutions. {\cn This isolates the effect of the data discretization by directly approximating $\inputdiscretization(u)$ and $\outputdiscretization(v)$. In this way, we focus on investigating the effects of discretization independence. Generating low-resolution data using a low-resolution solver (i.e., a multifidelity setting) would result in extra computational savings in the cost of generating data, but it would also introduce extra numerical error that would complicate our investigation of discretization independence.} 

We implement the neural networks $\encodernet$, $\approxnet$, and $\reconstructnet$ with multilayer perceptrons with the leaky ReLU activation function, using Fourier feature expansion \citep{lu_2022} for $\encodernet$ and $\reconstructnet$. 
See Appendix \ref{sec:appendix:training} for the training details and model architectures.   
We train all models on training datasets of size $N = 2048$ except for the Navier--Stokes models, where we use $N=4096$.  
The Adam optimization method with a geometrically attenuated learning rate is used to minimize the $L^1$ loss\footnote{We chose $L^1$ loss because it seemed to give slightly better training results; however, we note that our results also hold if we use $L^2$ or relative $L^2$ loss, the most common alternatives.}, which we approximate by a Riemann sum. 
All models are tested on single-{\cn resolution} ($M=1$) testing datasets of size $N=512$ with varying values of $R_1$ to assess discretization independence and other {\cn multiresolution} effects. To evaluate performance we compute the average relative $L^1$ error (i.e., mean absolute error) between the model prediction and the true output function, again approximated by a Riemann sum. We average errors over multiple training runs with random initial weights for the neural networks $\encodernet$, $\approxnet$, and $\reconstructnet$ as well as random reorderings of the training dataset.

\subsection{Nonlocal Poisson problems}
\label{section:Poisson}
\setcounter{equation}{0}

Consider the two-dimensional ($d = 2$) nonlocal Poisson equation with homogeneous Dirichlet boundary conditions: 
    \begin{align}
        \label{eq:poisson}
        \begin{cases}
            (-\Delta)^\fl{\ap}{2} u = f(\vec{x}), \ &\quad \vec{x} \in \Omega, \\
            u(\vec{x}) = 0, &\quad \vec{x} \in {\mathbb R}^2\backslash\Omega,
        \end{cases}
    \end{align}
where the fractional Laplacian $(-\Delta)^\fl{\ap}{2}$ for $\ap\in (0, 2)$ is a nonlocal operator \citep{Ros-Oton2014,Duo2018}. In our study, we consider the case $\ap = 1$. This operator has numerous applications, including water waves, quantum mechanics, and anomalous diffusion. 
If $f$ is continuous on $\Omega$, then it is well-known that the function $u$ is uniquely determined by $f$ \citep{Ros-Oton2014}.
We define the operator $\trueop$ such that ${u = \trueop(f)}$, which is the operator that we attempt to approximate. 
To obtain the training data, we generate the input functions $f$ using a random Fourier series \citep{raonic_2023}, and obtain the corresponding solutions $u = \trueop(f)$ by numerically solving \eqref{eq:poisson} at high resolution using the methods described in \citep{Zhou2024}. 

\smallskip
\noindent\textbf{Single-{\cn resolution} training. } 
We first conduct experiments by training models with only single-{\cn resolution} (i.e. $M = 1$) datasets to demonstrate the discretization independence of our method. 
To this end, all models use the same architecture hyperparameters (described in Appendix \ref{sec:appendix:training}), and training dataset sizes $N$, with only the resolution of the training data varying between models.
Figures \ref{fig:fpoisson_plots_input_same} and \ref{fig:fpoisson_plots_input_high} compare the predicted solutions and their errors from models with different training data resolution ($R_1$) values.
In Figure \ref{fig:fpoisson_plots_input_same}, the test input data has the same resolution as the training data, while in Figure \ref{fig:fpoisson_plots_input_high}, the test input data has a fixed resolution that is higher than the training data.
\begin{figure}[htb!]
    \centering   \includegraphics[width=0.98\linewidth]{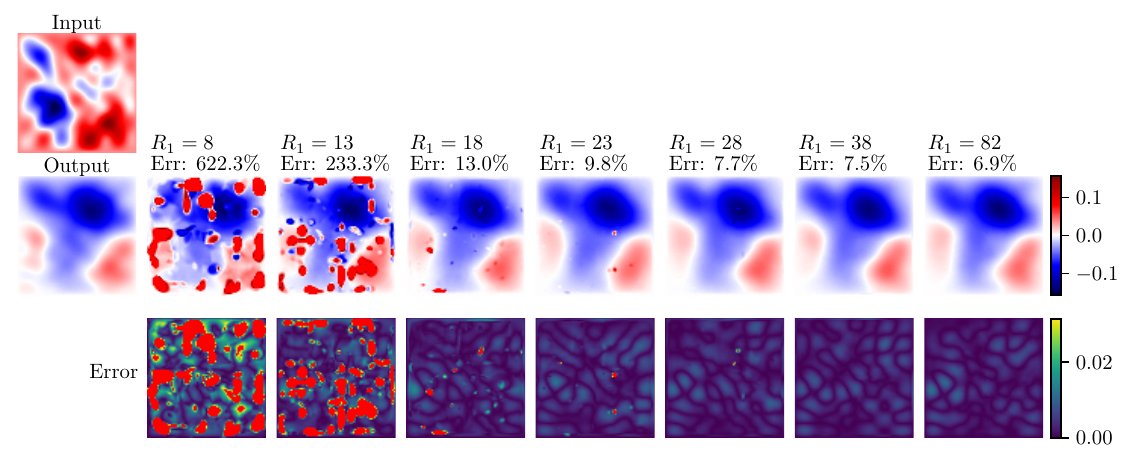}
    \vspace{-5mm}\caption{Effect of low-resolution \textbf{training outputs}. Sample predictions and errors of models with different training data resolutions\protect\footnotemark, with training resolution and relative $L^1$ error above each prediction. Inputs are interpolated to the \textbf{same resolution used for training} in order to isolate the effect of the training output resolution.} 
    \label{fig:fpoisson_plots_input_same}
    \vspace{-6mm}
\end{figure}
\footnotetext{In all heat maps shown here and below, values exceeding the color bar range are marked in bright red (above range) and dark blue (below range).}
From these results, we find that, when the training inputs have sufficiently high resolution (about $20 \times 20$ or higher), our method achieves good accuracy independent of the test data,  indicating its discretization independence.  
The discretization independence is further confirmed by the results in Table \ref{table:fpoisson_performance summary}, which also shows consistent test errors among models trained and tested at sufficiently high resolutions (roughly $R_1 \ge 23$ based on Table \ref{table:fpoisson_performance summary}).

Figures \ref{fig:fpoisson_plots_input_same} and \ref{fig:fpoisson_plots_input_high} further show that significant errors occur at very low training resolutions, when the input lacks enough information to make an accurate prediction of the output. 
\begin{figure}[htb!]
    \centering
    \includegraphics[width=0.98\linewidth]{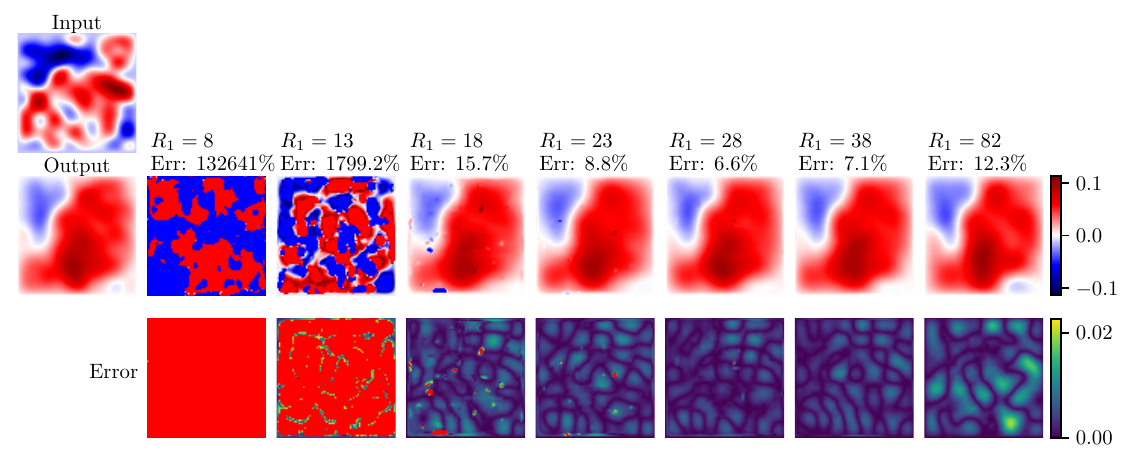}
    \vspace{-5mm}\caption{Effect of low-resolution \textbf{training inputs}. Sample predictions and errors of models with different training data resolutions for the fractional Poisson problem. All inputs have a \textbf{fixed, high resolution (82)}, which highlights the effects of the input training resolution.}
\label{fig:fpoisson_plots_input_high}
\vspace{-6mm}
\end{figure}
Interestingly, models trained on a very low resolution (e.g. $R_1 = 8$, or $13$) datasets still capture some features of the operator. 
They are, however, heavily dependent on the particular training input resolution, as we see their predictions blow up completely when tested on a higher-resolution input in Figure \ref{fig:fpoisson_plots_input_high}. 
This is a result of the separate effects of a low-resolution input on the encoder and a low-resolution output on the reconstructor and loss function. 

\begin{table}[htb!]
    \centering
    \begin{tabular}{@{}rrrrrrrr@{}}
        \toprule
        \multirow{2}{*}{\parbox{1.7cm}{Testing resolution}} & \multicolumn{7}{c}{Training resolution} \\
        \cmidrule(ll){2-8}
        & 8 & 13 & 18 & 23 & 28 & 38 & 82 \\
        \midrule
        8 & 6.9 & 1385.3 & 90.6 & 90.3 & 92.9 & 81.3 & 166.5 \\
        13 & 265253.6 & 7.9 & 14.4 & 12.2 & 12.2 & 11.7 & 12.2 \\
        18 & 260943.3 & 1218.0 & 8.5 & 10.3 & 9.7 & 9.9 & 10.6 \\
        23 & 265765.4 & 1620.5 & 18.6 & 9.5 & 8.5 & 8.7 & 9.1 \\
        28 & 263824.2 & 1437.7 & 21.5 & 9.8 & 8.1 & 8.2 & 9.1 \\
        38 & 263376.4 & 1525.0 & 21.2 & 11.1 & 8.6 & 8.4 & 8.6 \\
        82 & 259087.9 & 1466.6 & 20.4 & 10.8 & 8.8 & 8.6 & 8.5 \\
        \bottomrule
    \end{tabular}
    \vspace{-1mm}\caption{Relative $L^1$ error (in $\%$) for models trained and evaluated on single-{\cn resolution} datasets. Error is averaged across 8 training runs.}
    \label{table:fpoisson_performance summary}
    \vspace{-5mm}
\end{table}

In Figure \ref{fig:fpoisson_plots_input_same}, by using test inputs with the same resolution as used in training, we isolate the effects of the low resolution of the output training data. 
Evidently, training with a low-resolution output results in predictions that are generally consistent with the true output, but with local blow-ups away from the output sampling points; see the red regions in the errors plot of Figure \ref{fig:fpoisson_plots_input_same}. 
On the other hand, we observe from Figure \ref{fig:fpoisson_plots_input_high} that training with low-resolution input data could lead to large global errors (see the red regions in the errors plot of Figure  \ref{fig:fpoisson_plots_input_high}). Notably, the effects of  low-resolution input data appear to vanish at lower resolutions than the effects of low-resolution output data; in particular, we can see that the model trained on data with resolution 18 (input and output) exhibits errors due to output resolution being too low (Figure {\ref{fig:fpoisson_plots_input_same}}) but no errors due to input resolution being too low (Figure {\ref{fig:fpoisson_plots_input_high}}). 

\smallskip
\noindent\textbf{{\cn Multiresolution} training. } 
Our previous results show that our method can achieve discretization independence. Furthermore, training with {\cn purely low-resolution} data requires less training time but results in higher prediction errors. In contrast, training with high-resolution data leads to lower errors but incurs high computational costs. To leverage the advantages of both high and low resolutions, we conduct a second set of experiments that demonstrate that {\cn multiresolution} training can not only improve training efficiency but also effectively eliminate the local blow-ups caused by low training output resolution.

Everything in these experiments is identical to the previous experiments, except that we use {\cn multiresolution} training datasets with the same total size $N$ as the single-{\cn resolution} training datasets used previously. 
\begin{figure}[htb!]
    \centering
    \includegraphics[width=0.95\linewidth]{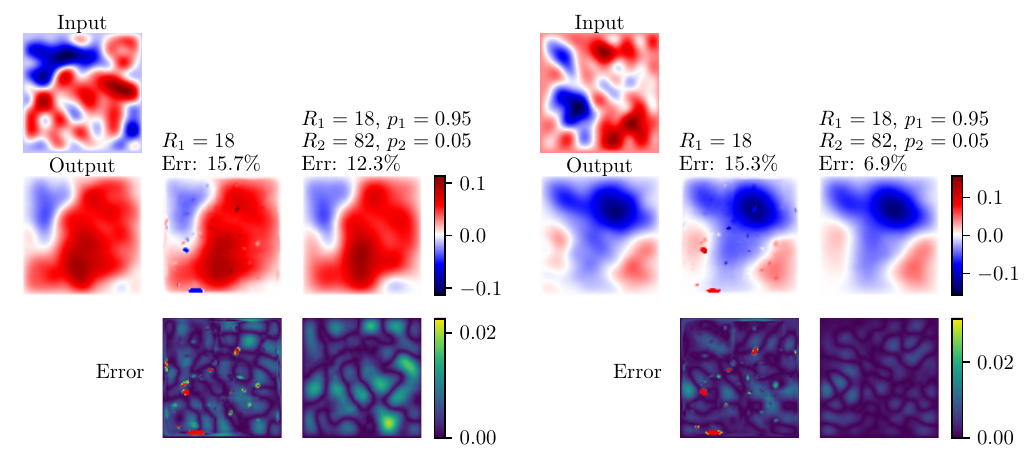}
    \vspace{-5mm}
    \caption{{\cn Multiresolution} training reduces errors caused by low-resolution training outputs. Predictions and errors for models trained on low-{\cn resolution} and {\cn multiresolution} training data, with training dataset parameters and relative $L^1$ errors above each prediction. All test inputs share a fixed, high resolution (i.e. 82).}  
    \label{fig:fpoisson_plots_multifidelity}
    \vspace{-6mm}
\end{figure}
Figure \ref{fig:fpoisson_plots_multifidelity} presents the {\cn multiresolution} training results, which may be compared with Figures \ref{fig:fpoisson_plots_input_same}--\ref{fig:fpoisson_plots_input_high}. 
It shows that training with mostly low-{\cn resolution} data and even a small fraction (5\%) of high-{\cn resolution} data results in accuracy similar to models trained with all high-{\cn resolution} data and completely removes the artifacts caused by low output resolution.
Furthermore, we see in Table \ref{table:fpoisson_multifidelity_summary} that error for models trained on {\cn multiresolution} datasets is comparable to that for models trained on single high-{\cn resolution} datasets, even when most of the data in the {\cn multiresolution} dataset is lower-{\cn resolution}. This fact can result in simultaneously lower error and lower data costs.

Indeed, for a two-dimensional problem such as this one, the average dimensions $n$ and $m$ of the discretized input and output spaces (i.e., the number of points at which the input and output functions are sampled) is given by
\begin{equation}
    \label{eq:average_data_size}
    \mathbb{E}[\inputdiscdimension] = \mathbb{E}[\outputdiscdimension] = \sum_{i=1}^M p_i R_i^2.
\end{equation}
Consequently, we see that the increased accuracy for $R_1=82$ over $R_1=18$, as observed in Table \ref{table:fpoisson_multifidelity_summary}, requires an increase in data and training costs by a factor of around 20.
In contrast, training with {\cn multiresolution} data that includes only a small fraction of high-{\cn resolution} data can achieve the same increase in accuracy by increasing the data {\cn size} and training costs only by a factor of 2, a tenfold improvement over the single-{\cn resolution} approach.
\begin{table}[htb!]
    \centering
    \begin{tabular}{@{}rrrrr@{}}
        \toprule
        \multirow{3}{*}{\parbox{1.7cm}{\vspace{1.5em}Testing resolution}} & \multicolumn{4}{c}{Training dataset parameters} \\
        \cmidrule(ll){2-5}
        & \multirow{2}{*}{$R_1 = 18$} & \multirow{2}{*}{$R_1 = 82$} & 
        $R_1 = 18$, $p_1=0.7$ & $R_1 = 18$, $p_1 = 0.95$ \\
        & && $R_2 = 82$,  $p_2=0.3$ & $R_2 = 82$,  $p_2=0.05$ \\
        \midrule
        13 & 14.4 & 12.2 & 9.7 & 9.4 \\
        18 & 8.5 & 10.6 & 8.2 & 7.9 \\
        38 & 21.2 & 8.6 & 7.6 & 8.2 \\
        82 & 20.4 & 8.5 & 7.5 & 7.9 \\
        \midrule
        Average data size & 324 & 6724 & 2244 & 644 \\
        \bottomrule
    \end{tabular}
    \vspace{-1mm}
    \caption{{\cn Multiresolution} training reduces error compared to single-{\cn resolution} training with a similar data cost. We report relative $L^1$ errors (in $\%$) for models trained on \textbf{both single- and {\cn multiresolution}} datasets and evaluated on single-{\cn resolution} datasets. Error is averaged across 5 training runs. Average data size refers to the average number of measurements used to represent a sample randomly chosen from the training dataset---training costs are proportional to average data size.}    \label{table:fpoisson_multifidelity_summary}
    \vspace{-8mm}
\end{table}

\subsection{Darcy flow problem}
\label{section:Darcy}
In addition to handling discretizations with different sampling resolutions, discretization independence enables straightforward handling of domains with non-rectangular geometries. For our model, specifically, discretization independence is achieved by numerical integration, so our method works for any domain geometry on which a numerical quadrature rule can be defined. In this section we demonstrate this advantage by learning the solution operator for Darcy flow on a triangular domain. 

Consider the two-dimensional Darcy flow equation on the interior $\Omega$ of the equilateral triangle with vertices $(0,0)$, $(1,0)$, and $(1/2, \sqrt{3}/2)$:
\begin{equation}
    \label{eq:darcy_triangle}
    \begin{cases}
        -\nabla\cdot (K(x,y)\nabla p(x,y)) = f(x,y), \quad & (x,y) \in \Omega, \\
        p(x,y) = 0,  & (x,y) \in \partial \Omega.
    \end{cases}
\end{equation}
Here, $K$ is the permeability field, $p$ is the pressure field, and $f$ is a source term. For a fixed source $f$, the pressure $p$ is determined uniquely by the permeability field $K$. We define the target operator $G^\dagger$ by $G^\dagger(K) = p$. In the following, we fix $f(x,y) = -0.1$. 

The training and testing data are generated by sampling the permeability $K$ from a random Fourier series defined on the square $[0,1]^2$, as in Section \ref{section:Poisson}, but we then restrict to the triangular domain $\Omega$. Additionally, we add a fixed positive number $K_0$ to the generated field to ensure that $K$ is positive on $\Omega$, which is required for \eqref{eq:darcy_triangle} to be well-posed. To obtain the pressure $p$ we solve \eqref{eq:darcy_triangle} numerically using the finite element method with piecewise linear elements on a uniform equilateral triangular mesh with mesh size $1/256$. We use the assembly library \texttt{scikit-fem} \citep{skfem2020} to implement our solver. 
Like in the other examples, we also evaluate the input $K$ at the same points as the output. Since this discretization gives the function values at the vertices of triangular elements, we perform numerical integration by compositing a rule on each element that is exact for linear functions (i.e., with convergence order analogous to the trapezoidal rule for rectangular grids). For this experiment, we use the dataset parameter $R_i$ to refer to the number of sampling points along one edge of $\Omega$.

\smallskip
\noindent{\textbf{Single-{\cn resolution} training. } We start with single-{\cn resolution} experiments, training models on various single-{\cn resolution} datasets and testing them at different resolutions. 
Figure \ref{fig:darcy_plots_input_high} presents the predicted solutions and their errors from models with different training data resolutions, while Table \ref{table:darcy_performance summary} summarizes the relative errors for different combinations of training and testing resolutions. Similar to the observations in Section \ref{section:Poisson}, when the training inputs have sufficiently high resolution (e.g., $R_1\ge 65$), our method achieves good accuracy independent of the test resolution. Moreover, these results show that our method is effective for problems on non-rectangular domains due to its discretization independence.
\begin{figure}[htb!]
    \centering
    \includegraphics[width=0.9\linewidth]{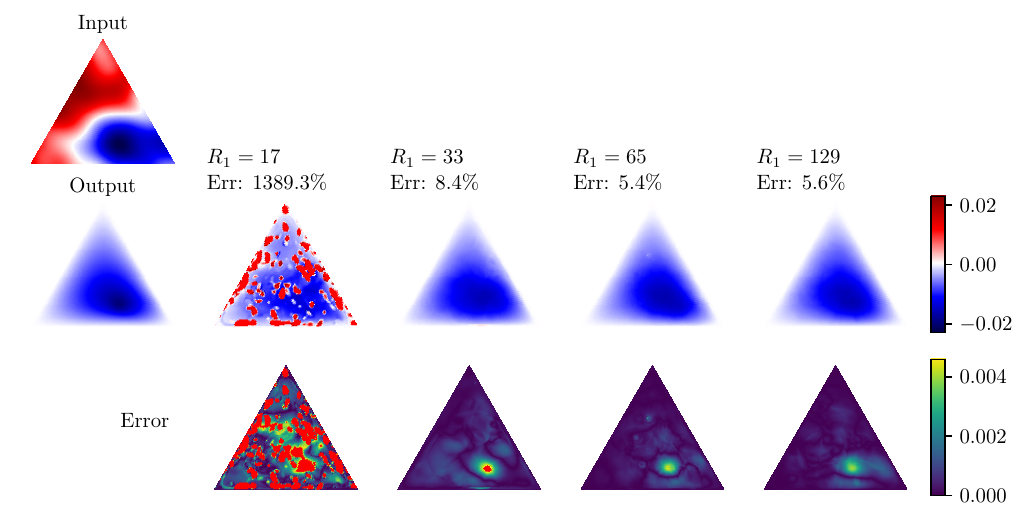}
    \vspace{-6mm}
    \caption{Effect of low-resolution \textbf{training inputs} on a triangular domain. Sample predictions and errors of models with different training data resolutions for the Darcy flow problem. All inputs have a \textbf{fixed, high resolution (129)}, which highlights the effects of the input training resolution.}  
    \label{fig:darcy_plots_input_high}
    \vspace{-6mm}
\end{figure}
\begin{table}[htb!]
    \centering
    \begin{tabular}{@{}rrrrr@{}}
        \toprule
        \multirow{2}{*}{\parbox{1.7cm}{Testing resolution}} & \multicolumn{4}{c}{Training resolution} \\
        \cmidrule(ll){2-5}
        & 17 & 33 & 65 & 129 \\
        \midrule
        17 & 9.2 & 23.4 & 2227.3 & 1388.6 \\
        33 & 3303.8 & 8.8 & 159.4 & 296.5 \\
        65 & 2920.7 & 10.0 & 6.0 & 15.6 \\
        129 & 2935.7 & 9.7 & 6.0 & 4.8 \\
        \bottomrule
    \end{tabular}
    \vspace{-1mm}\caption{Relative $L^1$ error (in $\%$) for models trained and evaluated on single-{\cn resolution} datasets. Error is averaged across 8 training runs.}
    \label{table:darcy_performance summary}
\end{table}

\begin{figure}[htb!]
    \centering
    \includegraphics[width=0.95\linewidth]{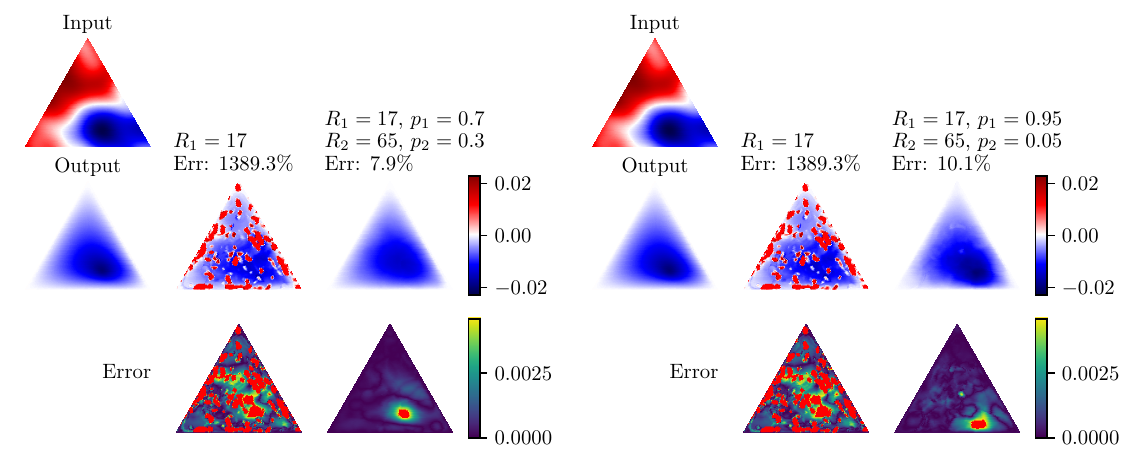}
    \vspace{-5mm}
    \caption{{\cn Multiresolution} training reduces errors caused by low-resolution training outputs. Predictions and errors for models trained on low-{\cn resolution} and {\cn multiresolution} training data, with training dataset parameters and relative $L^1$ errors above each prediction. All test inputs share a fixed, high resolution (i.e. 129).}  
    \label{fig:darcy_plots_multifidelity}
\end{figure}

\begin{table}[htb!]
    \centering
    \begin{tabular}{@{}rrrrr@{}}
        \toprule
        \multirow{3}{*}{\parbox{1.7cm}{\vspace{1.5em}Testing resolution}} & \multicolumn{4}{c}{Training dataset parameters} \\
        \cmidrule(ll){2-5}
        & \multirow{2}{*}{$R_1 = 17$} & \multirow{2}{*}{$R_1 = 65$} & 
        $R_1 = 17$, $p_1=0.7$ & $R_1 = 17$, $p_1 = 0.95$ \\
        & && $R_2 = 65$,  $p_2=0.3$ & $R_2 = 65$,  $p_2=0.05$ \\
        \midrule
        17 & 9.2 & 2227.3 & 7.6 & 8.2 \\
        33 & 3303.8 & 159.4 & 6.6 & 6.2 \\
        65 & 2920.7 & 6.0 & 6.4 & 6.0 \\
        129 & 2935.7 & 6.0 & 6.3 & 5.7 \\
        \midrule
        Average data size & 153 & 2145 & 751 & 253 \\
        \bottomrule
    \end{tabular}
    \vspace{-1mm}
    \caption{{\cn Multiresolution} training reduces error compared to single-{\cn resolution} training with a similar data cost. We report relative $L^1$ errors (in $\%$) for models trained on \textbf{both single- and {\cn multiresolution}} datasets and evaluated on single-{\cn resolution} datasets. Error is averaged across 5 training runs. }    \label{table:darcy_multifidelity_summary}
    \vspace{-8mm}
\end{table}

\smallskip
\noindent {\textbf{{\cn Multiresolution} training.} Figure \ref{fig:darcy_plots_multifidelity} and Table \ref{table:darcy_multifidelity_summary} present the results for {\cn multiresolution} training. All other settings are the same as in the single-{\cn resolution} case, except that {\cn multiresolution} training  data are used. Figure \ref{fig:darcy_plots_multifidelity} shows that including 5\% medium-{\cn resolution} data significantly improves accuracy, reducing the relative error from 1390.3\% to 10.1\%.
Similar to our observations in Section \ref{section:Poisson}, training with {\cn multiresolution} data that includes only a small fraction of high-{\cn resolution} samples can achieve accuracy comparable to using only high-{\cn resolution} data, but the average data size is much smaller (see Table \ref{table:darcy_multifidelity_summary}). Since the training cost is proportional to the average data size, this leads to a significant reduction in training cost. 

\subsection{Viscous Burgers' equation}
\label{section:Burgers}

In this section we learn the solution operator for nonlinear PDEs. 
Consider the one-dimensional viscous Burgers' equation with periodic boundary conditions: 
    \begin{equation}
        \label{eq:burgers}
        \begin{cases}
            u_t(x, t) + \frac{1}{2}(u^2)_x = \nu u_{xx}, \quad \  & x\in \Omega, \ \ t\in(0, T], \\
            u(x, 0) = u_0(x), \quad &  x\in\overline{\Omega},
        \end{cases}
    \end{equation}
where the viscosity parameter $\nu > 0$. 
For a fixed viscosity $\nu$, the solution $u(\cdot, T)$ at any time $T > 0$ is determined by the initial condition $u_0$. 
Therefore, we can define an operator $\trueop$ such that $\trueop(u_0) = u(\cdot, T)$. 
In our study, we choose $\nu = 0.005$ and $T = 1$. 

The training and testing data are generated by sampling the initial condition $u_0(x)$ from a Gaussian process $N(0, \; \sigma^2(-\Delta + \tau^2 I)^{-1})$ on $\Og$ with periodic boundary conditions \citep{li_2021a}. 
In this case, we choose $\sigma = 5$ and  $\tau=5$ and denote $\Dt = \frac{d^2}{dx^2}$. 
We then numerically solve \eqref{eq:burgers} using a Fourier pseudospectral method for spatial discretization and the fourth-order Runge--Kutta method for time integration. 

\smallskip
\noindent\textbf{Single-{\cn resolution} training. } 
We start with single-{\cn resolution} experiments, training models on various single-{\cn resolution} datasets and testing them at different resolutions. Figure \ref{fig:burgers_plots_input_high} compares the predicted solutions with the ground truth solutions for three sample inputs, with models trained on single-{\cn resolution} datasets of different resolutions. 
\begin{figure}[htb!]
    \centering
    \includegraphics[width=0.9\linewidth]{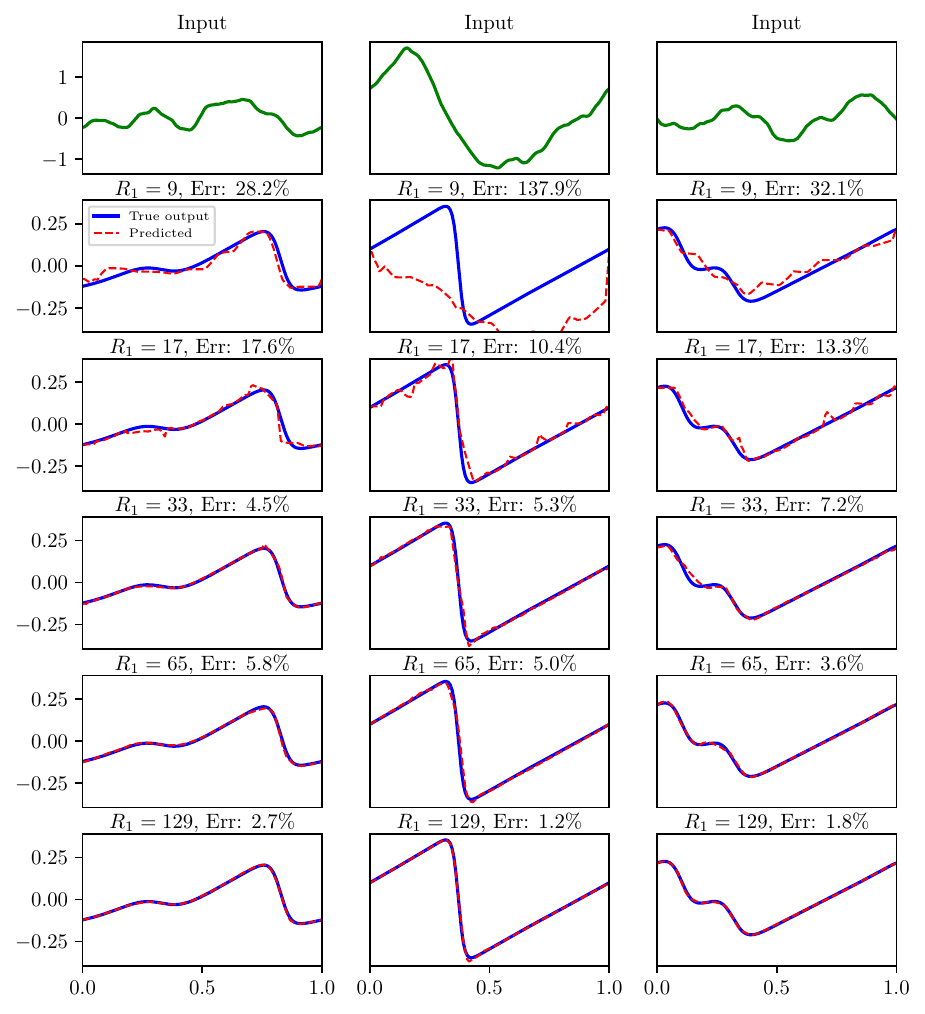}
    \vspace{-5mm}
    \caption{Sample predictions of models with different training data resolutions for the Burgers problem. All test inputs have the same high resolution (i.e. $129$).}
    \label{fig:burgers_plots_input_high}
    \vspace{-6mm}
\end{figure}
It shows that the error is significantly exacerbated when the training resolution is too low. For instance, when training with $R_1 = 9$, high errors are observed in both the left and right examples, while the middle example experiences a blow-up. 
However, for sufficiently high-resolution training data, our method demonstrates empirical discretization independence, with consistent errors that are independent of the training and testing resolutions. 
This is further confirmed by the results presented in Table \ref{table:burgers_performance_summary}. 

In Table \ref{table:burgers_performance_summary}, we summarize the average errors for models trained on single-{\cn resolution} datasets and tested on datasets with different resolutions. 
\newcommand{\ovl}[1]{{\color{BlueViolet}#1}}
\newcommand{\novl}[1]{#1}
{\small \begin{table}[htb!]
    \centering
        \small
    \begin{tabular}{@{}rrrrrrrrrrr@{}}
        \toprule
        
        \multirow{2}{*}{\parbox{1.6cm}{Testing resolution}} & \multicolumn{10}{c}{Training resolution} \\
        \cmidrule(ll){2-11}

         & 11 & 17 & 26 & 33 & 51 & 65 & 129 & 201 & 513 & 1025 \\

        \midrule

        
        11  & \ovl{1.5} & \novl{138.1} & \ovl{1044.2} & \novl{2285.0} & \ovl{14163.1} & \novl{11496.0} & \novl{4282.8} & \ovl{2942.8} & \novl{2481.8} & \novl{787.0} \\
        
        17  & \novl{77.5} & \ovl{1.7} & \novl{67.1} & \ovl{97.7} & \novl{1956.8} & \ovl{1811.1} & \ovl{478.9} & \novl{275.9} & \ovl{205.8} & \ovl{40.6} \\
        
        26  & \ovl{58.4} & \novl{21.3} & \ovl{3.2} & \novl{12.7} & \ovl{123.0} & \novl{151.5} & \novl{30.6} & \ovl{15.1} & \novl{26.0} & \novl{14.1} \\
        
        33  & \novl{74.4} & \ovl{14.0} & \novl{11.4} & \ovl{3.2} & \novl{12.6} & \ovl{13.9} & \ovl{6.6} & \novl{11.3} & \ovl{11.1} & \ovl{6.4} \\
        
        51  & \ovl{62.2} & \novl{20.0} & \ovl{6.8} & \novl{10.9} & \ovl{3.2} & \novl{10.1} & \novl{9.4} & \ovl{4.0} & \novl{10.5} & \novl{9.4} \\
        
        65  & \novl{80.4} & \ovl{16.8} & \novl{11.0} & \ovl{10.2} & \novl{9.4} & \ovl{3.4} & \ovl{3.5} & \novl{8.8} & \ovl{4.5} & \ovl{3.6} \\
        
        129  & \novl{81.7} & \ovl{17.3} & \novl{11.1} & \ovl{6.8} & \novl{9.3} & \ovl{3.7} & \ovl{3.2} & \novl{8.8} & \ovl{3.5} & \ovl{3.3} \\
        
        201  & \ovl{63.4} & \novl{19.6} & \ovl{7.3} & \novl{10.9} & \ovl{4.0} & \novl{9.7} & \novl{9.1} & \ovl{3.2} & \novl{9.0} & \novl{9.0} \\
        
        401  & \ovl{63.3} & \novl{19.6} & \ovl{7.3} & \novl{10.9} & \ovl{4.0} & \novl{9.7} & \novl{9.1} & \ovl{3.2} & \novl{9.0} & \novl{9.0} \\
        
        513 & \novl{81.9} & \ovl{17.6} & \novl{11.1} & \ovl{6.2} & \novl{9.4} & \ovl{3.7} & \ovl{3.3} & \novl{8.8} & \ovl{3.4} & \ovl{3.2} \\
        
        1025 & \novl{82.0} & \ovl{17.6} & \novl{11.1} & \ovl{6.2} & \novl{9.4} & \ovl{3.7} & \ovl{3.3} & \novl{8.8} & \ovl{3.4} & \ovl{3.2} \\

        \bottomrule
    \end{tabular}
    \vspace{-1mm}\caption{Discretization independence and bias for Burgers problem. We present relative $L^1$ testing errors (in $\%$) for single-{\cn resolution} Burgers experiments. \ovl{\textbf{Purple}} \textbf{numbers} indicate that the test discretization and training discretization have \textbf{significantly overlapping sampling points} whereas white cells have almost no overlap.} 
    \label{table:burgers_performance_summary}
    \vspace{-5mm}
\end{table}
}
These testing resolutions are intentionally selected to feature either many overlapping sample points or nearly disjoint sample points. 
For instance, the \ovl{purple} errors indicate that the training and testing resolutions share many identical sampling points, while the black errors indicate that the training and testing resolutions have almost no overlapping sampling points.
Table \ref{table:burgers_performance_summary} shows that while the test error remains consistent and acceptably low at sufficiently high training resolutions, there is a noticeable trend toward lower test error when the training and testing discretizations share many overlapping sample points. 
We call this effect \emph{discretization bias}. 

Table \ref{table:burgers_performance_summary} further shows that, for a fixed training resolution, the error difference between two testing resolutions---one with many overlapping sample points and the other with no overlapping sampling points---remains roughly constant\footnote{Assuming all have sufficiently high resolution to avoid low-resolution issues.}.
Hence, we propose the following metric to measure discretization bias, which we call the \emph{performance gap}. 
Given a collection of single-{\cn resolution} testing datasets $\mathcal{D}_1, \mathcal{D}_2, \dots \mathcal{D}_K$, let $S_\text{same}$ be the set of indices of datasets whose sampling points highly overlap with those of one of the training discretizations, and let $S_\text{diff}$ be the set of indices of the datasets whose sampling points do not. Then we define the performance gap as the difference in average testing error between test datasets with overlap and those without overlap: 
\begin{equation}
    \label{eq:performance_gap}
    g = \frac{1}{|S_\text{same}|}\sum_{j \in S_\text{same}} E_j - \frac{1}{|S_\text{diff}|}\sum_{j \in S_\text{diff}} E_j,
\end{equation}
where $E_j$ is the test error (in our case we use mean relative $L^1$ error) on dataset $\mathcal{D}_j$.
We note that the smaller the performance gap $g$, the more discretization-independent the model is. See more illustration and discussion in Table \ref{table:burgers_performance_gap}. 

\smallskip
\noindent\textbf{{\cn Multiresolution} training. } 
Our earlier studies show that {\cn multiresolution} training can effectively reduce discretization dependence caused by low-resolution training outputs. 
Hence, we hypothesize that {\cn multiresolution} training can also mitigate discretization bias. 
To test this, we conduct experiments by training models using {\cn multiresolution} data. 
In this case, we specifically track which models are exposed to which sampling points in their training data. 
Consequently, we can compute and compare the performance gap for selected models in both the single-{\cn resolution} and {\cn multiresolution} settings.

In Table \ref{table:burgers_multifidelity_detailed} we present the test errors for models trained on {\cn multiresolution} datasets. 
\begin{table}[htb!]
\centering
\begin{tabular}{@{}rrrcrc@{}}
    \toprule
    \multirow{3}{*}{\parbox{1.7cm}{\vspace{1.5em}Testing resolution}}& \multirow{3}{*}{\parbox{1.7cm}{\vspace{1.5em}Power /multiple}}& \multicolumn{4}{c}{Training dataset parameters} \\
    \cmidrule(ll){3-6}
    & & \multicolumn{2}{l}{$R_1=51$, $R_2=65$} & \multicolumn{2}{l}{$R_1=51$, $R_2=65$, $R_3=82$} \\
    \cmidrule(ll){3-4} \cmidrule(ll){5-6} 
     & & Abs. error & Rel. error (\%) &  Abs. error & Rel. error (\%) \\
    \midrule
    26 & 5 & 1.49e-2 & \novl{10.2\hspace{2mm}} & 1.11e-2 & \novl{7.3}\\
    28 & 3 & 1.47e-2 & \novl{9.8} & 8.32e-3 & \ovl{5.2}\\
    33 & 2 & 7.73e-3 & \ovl{4.8}& 7.49e-3 & \ovl{4.6}\\
    50 & 7 & 1.09e-2 & \novl{7.4} & 9.65e-3 & \novl{6.4}\\
    51 & 10 & 6.40e-3 & \ovl{3.9} & 6.68e-3 & \ovl{4.1}\\
    65 & 2 & 6.50e-3 & \ovl{3.9} & 6.84e-3 & \ovl{4.1}\\
    82 & 3 & 1.14e-2 & \novl{7.8} & 6.81e-3 & \ovl{4.1}\\
    126 & 5 & 1.11e-2 & \novl{7.5} & 9.53e-3 & \novl{6.3}\\
    129 & 2 & 6.74e-3 & \ovl{4.0} & 6.97e-3 & \ovl{4.2}\\
    201 & 10 & 6.87e-3 & \ovl{4.2} & 6.90e-3 & \ovl{4.2}\\
    244 & 3 & 1.14e-2 & \novl{7.8} & 6.91e-3 & \ovl{4.2}\\
    344 & 7 & 1.08e-2 & \novl{7.3} & 9.64e-3 & \novl{6.4}\\
    401 & 10 & 6.88e-3 & \ovl{4.2} & 6.91e-3 & \ovl{4.2}\\
    513 & 2 & 6.77e-3 & \ovl{4.1} & 6.99e-3 & \ovl{4.2}\\
    626 & 5 & 1.11e-2 & \novl{7.5} & 9.53e-3 & \novl{6.3}\\
    730 & 3 & 1.14e-2 & \novl{7.8} & 6.92e-3 & \ovl{4.2}\\
    1025 & 2 & 6.77e-3 & \ovl{4.1} & 6.99e-3 & \ovl{4.2}\\
    \bottomrule
\end{tabular}
\vspace{-1mm}\caption{Effect of {\cn multiresolution} training on discretization bias. We report average errors across 10 runs of two models trained on {\cn multiresolution} Burgers equation data with equal mixtures of discretizations (i.e., $p_i = \frac{1}{M}$). Each model used the same total amount of training data $N$, which was the same as in Table \ref{table:burgers_performance_summary}.  \ovl{\textbf{Purple}} \textbf{numbers} indicate testing resolutions whose sampling points \textbf{significantly overlap} those of one of the training resolutions whereas others have almost no overlap. Resolutions with the same value in the Power/multiple column have significantly overlapping sampling points, and those with different values have very few overlapping points.}
\label{table:burgers_multifidelity_detailed}
\vspace{-6mm}
\end{table}
The table shows that {\cn multiresolution} training reduces error variability, with the more diverse {\cn multiresolution} training data (i.e., more discretizations) leading to a more significant reduction. 
Importantly, more diverse {\cn multiresolution} training reduces the error even for testing discretizations whose sampling points are not seen during training. 
This observation is further summarized in Table \ref{table:burgers_performance_gap}, which demonstrates clearly that increasing the diversity of the {\cn multiresolution} training data reduces the performance gap, thus confirming our hypothesis that {\cn multiresolution} training reduces discretization bias.

\begin{table}[htb!]
    \centering
    \begin{tabular}{@{}lll@{}}
        \toprule
        Training resolutions & Proportions & Performance gap\\
        \midrule
        51 & 1 & 5.8 \\[0.25em]
        65 & 1 & 5.9 \\[0.25em]
        82 & 1 & 5.7 \\[0.25em]
        \midrule 
        51, 65 & $\frac{1}{2}$, $\frac{1}{2}$ & 3.5 \\[0.25em]
        51, 65, 82 & $\frac{1}{3}$, $\frac{1}{3}$, $\frac{1}{3}$ & 2.2 \\
        \bottomrule
    \end{tabular}
    \vspace{-1mm}
    \caption{Discretization performance gap in relative $L^1$ error (\%) for models trained with different mixtures of training data---but still the same total amount of training data. Note that all models achieved roughly the same accuracy for resolutions whose sampling points were seen during training. Reported 
    performance gap is the average over 10 runs.}
    \label{table:burgers_performance_gap}
    \vspace{-6mm}
\end{table}

Finally, it is important to note that discretization bias is not an inherently bad property. While it indicates that the model is mildly dependent on the discretization, this dependency manifests as an \emph{increase} in accuracy on certain discretizations, surpassing the baseline accuracy on others (with the magnitude of the accuracy improvement quantified by the performance gap), as seen in Table \ref{table:burgers_performance_gap}.

\subsection{Navier--Stokes problems}
\label{section:NS}

In this example we test our method on the Navier--Stokes equation. 
We again show that our method is discretization independent and further confirm, this time for a nonlinear problem, that {\cn multiresolution} training effectively reduces errors due to low-resolution  training outputs, as demonstrated in the fractional Poisson example. 
Finally, we discuss a new phenomenon that arises when models trained on high-resolution data are tested on very low-resolution data. 

We consider the two-dimensional incompressible Navier--Stokes equations for the vorticity $\omega(x,y,t)$ with periodic boundary conditions, which are given by
\begin{equation}
    \label{eq:navier_stokes}
    \begin{cases}
        \omega_t + \vec{u} \cdot \nabla \omega = \nu \Delta \omega + f, &\quad  (x,y) \in \Omega,\, t \in (0, T], \\
        \nabla \cdot \vec{u} = 0, &\quad (x,y) \in \Omega,\, t \in (0, T],\\
        \omega(x, y, 0) = \omega_0(x, y) &\quad (x,y) \in \overline{\Omega};
    \end{cases}
\end{equation}
where $\vec{u} = \left(u_1, u_2, 0\right)^T$ is the velocity, so that $\nabla \times \vec{u} = (0, 0, \omega)^T$. 
Note that the velocity $\vec{u}$ can be obtained from $\omega$ by finding a stream function $\psi$; that is, $\omega = -\Delta \psi$, subject to periodic boundary conditions.
Then $\vec{u}$ can be found by differentiation: $\vec{u} = (\psi_y, -\psi_x, 0)^T$. 
For a fixed viscosity value $\nu > 0$, the solution $\omega(x, y, t)$ at a time $t > 0$ is determined by the initial condition $\omega_0$. 
Thus, we define an operator $\trueop$ by $\trueop(\omega_0) = \omega(\cdot, \cdot, T)$, where $T> 0$ is a fixed final time. 
We approximate this operator with viscosity $\nu = 0.001$ and time $T = 2.2$. 

The training and testing data are generated  by sampling the initial condition $\omega_0(x,y)$ from a Gaussian process $
N\big(0, \, \sigma^2(-\Delta + \tau^2)^{-\alpha/2}\big)$ on $\Og$ with periodic boundary conditions. 
The parameters are chosen in all cases to be $\alpha = 2.5$, $\tau = 7$, and $\sigma = 7^\frac{3}{4}$. 
We then solve (\ref{eq:navier_stokes}) numerically for the solution $\omega(x,y,T)$, using the same numerical method and implementation as in \citep{li_2021a}. 

\smallskip
\noindent\textbf{Single-{\cn resolution} training. } 
Once again, we train models with identical architectures on various single-{\cn resolution} datasets and then evaluate their performance on single-{\cn resolution} test datasets.
The results of these experiments are summarized in Table \ref{table:navier_stokes_performance_summary}, in which we consistently observe discretization independence of our method at sufficiently high resolutions (e.g. $R_1 > 33$). 
\begin{table}[htb!]
    \centering
    \begin{tabular}{@{}rrrrrrr@{}}
        \toprule
        \multirow{2}{*}{\parbox{1.7cm}{Testing resolution}} & \multicolumn{6}{c}{Training resolution} \\
        \cmidrule(ll){2-7}
        & 16 & 17 & 28 & 33 & 46 & 65 \\
        \midrule
        9 & 6521.7 & 17.3 & 24.6 & 23.2 & 22.6 & 22.0\\
        16 & 10.6 & 32.8 & 15.8 & 12.9 & 12.4 & 12.0\\
        17 & 780.8 & 11.9 & 15.5 & 12.6 & 11.9 & 11.7\\
        28 & 64243.1 & 40.4 & 14.1 & 11.8 & 10.9 & 10.9\\
        33 & 16426.6 & 46.8 & 14.6 & 11.6 & 10.8 & 10.8\\
        46 & 24168.8 & 40.4 & 15.0 & 12.0 & 10.8 & 10.9\\
        65 & 26430.4 & 40.4 & 14.6 & 11.7 & 10.6 & 10.7\\
        129 & 17828.0 & 40.7 & 14.9 & 12.0 & 11.0 & 10.9\\
        136 & 16727.9 & 40.5 & 14.9 & 11.8 & 10.8 & 10.8\\
        \bottomrule
    \end{tabular}
    \vspace{-1mm}
    \caption{Relative $L^1$ error (in $\%$) for models trained and evaluated on single-{\cn resolution} datasets. Median error (in this problem there were a few significant outliers that caused the mean to be an unreliable measure) is averaged across 8 training runs.}\label{table:navier_stokes_performance_summary}
    \vspace{-6mm}
\end{table}
Sample results are visualized in Figure \ref{fig:navier_stokes_plots_input_same}, which further indicates the discretization independence of our method and shows that the same low-resolution effects from the fractional Poisson problem are also present in the Navier--Stokes problem. 
\begin{figure}[htb!]
    \centering
    \includegraphics[width=0.98\linewidth]{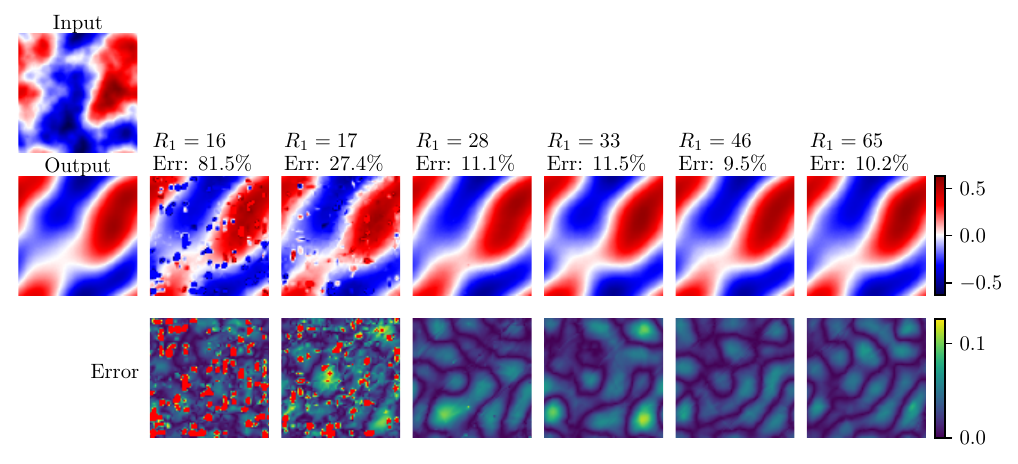}
    \vspace{-5mm}
    \caption{Effect of low-resolution {\bf training outputs}. Sample predictions and errors of models with different training data resolutions, with training resolution and relative $L^1$ error above each prediction. Inputs are interpolated to the \textbf{same resolution used for training}.}\label{fig:navier_stokes_plots_input_same}
    \vspace{-6mm}
\end{figure}

Moreover, our experiments show that if \textit{test} data with too low a resolution is used--rather than training data with too low resolution-- then all models will have higher error if they were trained at a different resolution, even if the training resolutions were high. 
This phenomenon can be seen in Table \ref{table:navier_stokes_performance_summary} and is further demonstrated in Figure \ref{fig:navier_stokes_plots_input_low}, in which we observe that testing with low-resolution inputs introduces larger errors for models trained on different resolution inputs.
This again suggests that very low-resolution inputs lack sufficient information to make accurate predictions. 
\begin{figure}[htb!]
    \centering
    \includegraphics[width=\linewidth]{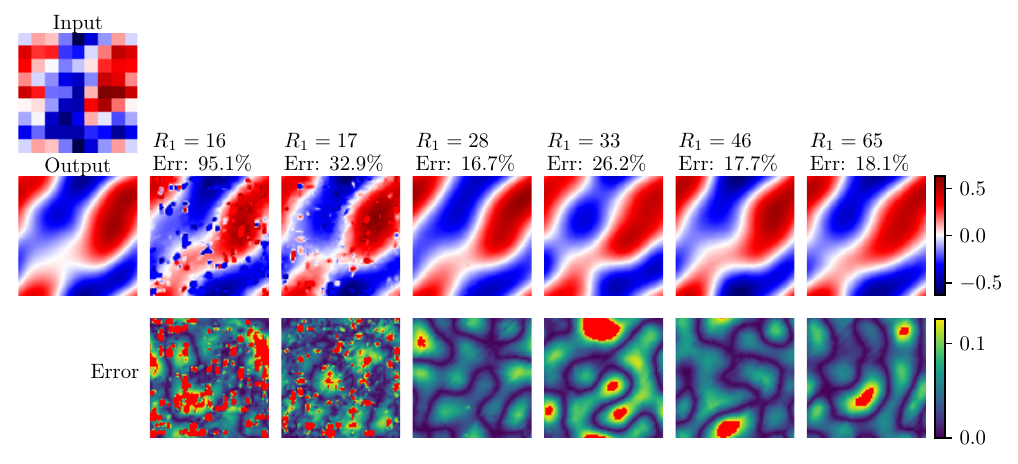}
    \vspace{-8mm}
    \caption{Effect of low-resolution \textbf{test input}. Sample predictions and errors of models with different training data resolutions. All inputs have a \textbf{fixed, low resolution} (i.e. 9), which highlights the effects of the input testing resolution.}\label{fig:navier_stokes_plots_input_low}
    \vspace{-9mm}
\end{figure}

\smallskip
\noindent\textbf{{\cn Multiresolution} training. }
We conduct similar {\cn multiresolution} experiments to those  for the fractional Poisson example. 
As shown in Figure \ref{fig:navier_stokes_plots_multifidelity}, we observe that {\cn multiresolution} training with even a small fraction of high-{\cn resolution} data yields the same benefits in this problem as it did in the fractional Poisson problem. 
Specifically, it eliminates the localized errors caused by training with low-{\cn resolution} data without incurring the significantly greater costs of training with only high-{\cn resolution} data. 

In Table \ref{table:navier_stokes_multifidelity_summary}, we further analyze the effect of {\cn multiresolution} training on reducing error and discretization bias. 
Comparing the single-{\cn resolution} case with $R_1 = 17$ to the {\cn multiresolution} case with $R_1=17$, $p_1=0.95$ and $R_2=65$, $p_2=0.05$ in Table \ref{table:navier_stokes_multifidelity_summary}
suggests that {\cn multiresolution} training even with a small fraction of high-resolution data significantly reduces error due to output resolution being too low. 
On the other hand, discretization bias seems to be negligible in this problem, as the performance gap is nearly zero for all training datasets.
As shown in Table \ref{table:navier_stokes_multifidelity_summary}, the performance gap is less than 0.1\% regardless of the training dataset. 

\begin{figure}[htb!]
    \centering
    \includegraphics[width=\linewidth]{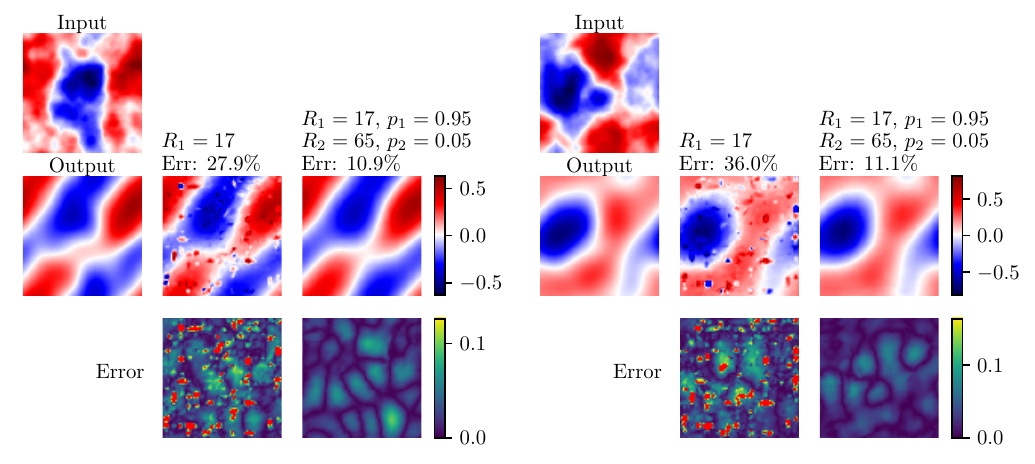}
    \vspace{-9mm}
    \caption{{\cn Multiresolution} training reduces errors caused by low-resolution training outputs. Predictions and errors for models trained on low-{\cn resolution} and {\cn multiresolution} training data, with training dataset parameters and relative $L^1$ errors above each prediction. All test inputs share a fixed, high resolution (i.e. 65).}\label{fig:navier_stokes_plots_multifidelity}
    \vspace{-5mm}
\end{figure}
\begin{table}[htb!]
    \centering
    \begin{tabular}{@{}rrrrr@{}}
        \toprule
        \multirow{3}{*}{\parbox{1.7cm}{\vspace{1.5em}Testing resolution}} & \multicolumn{4}{c}{Training dataset parameters} \\
        \cmidrule(ll){2-5}
        & \multirow{2}{*}{$R_1 = 17$} & \multirow{2}{*}{$R_1 = 65$} & 
        $R_1 = 17$, $p_1=0.7$ & $R_1 = 17$, $p_1 = 0.95$ \\
        & && $R_2 = 65$,  $p_2=0.3$ & $R_2 = 65$,  $p_2=0.05$ \\
        \midrule
        28 & 40.4 & 10.9 & 11.3 & 12.1 \\
        33 & 46.8 & 10.8 & 11.2 & 12.1 \\
        46 & 40.4 & 10.7 & 11.4 & 12.2 \\
        129 & 40.7 & 10.9 & 11.6 & 12.4 \\
        \midrule
        Performance gap & 0.02 & 0.06 & -0.03 & 0.00 \\
        \midrule
        Average data size & 289 & 4225 & 1470 & 486 \\
        \bottomrule
    \end{tabular}
    \vspace{-1mm}
    \caption{{\cn Multiresolution} training reduces error compared to single-{\cn resolution} training with a similar data cost. We report relative $L^1$ errors (in $\%$) and performance gaps (for test $R_1 \ge 28$) for models trained on \textbf{both single- and {\cn multiresolution}} datasets and evaluated on single-{\cn resolution} datasets. Error and performance gap are averaged across 8 training runs.} \label{table:navier_stokes_multifidelity_summary}
    \vspace{-8mm}
\end{table}

\section{Conclusion and discussion}
\label{section:conclusion}

In this paper we presented new concepts of numerical operator learning and discretization independence.
Based on these concepts, we proposed a novel encode-approximate-reconstruct model, about which we proved approximation guarantees using established theory for these types of models while retaining discretization independence, which existing methods do not.
Our experiments showed that our method is empirically discretization independent on a diverse set of operators.
We were also able to observe phenomena related to {\cn multiresolution} data; namely, we observed three types of low-resolution breakdowns, due to training inputs, training outputs, and testing inputs. 

Furthermore, we demonstrated that our model can leverage {\cn multiresolution} data to reduce low-resolution breakdowns and achieve good accuracy with lower data costs.
By establishing an understanding of how low-{\cn resolution} data introduces error, we were able to show with our experiments that {\cn multiresolution} training can reduce some types of error caused by low-{\cn multiresolution} training data, enabling faster training with lower-{\cn multiresolution} datasets.
Additionally, we observed the new phenomenon of discretization bias in the case of the Burgers example, and we showed how {\cn multiresolution} training mitigates it.

\acks{This work was partially supported by the Missouri University of Science and Technology's Kummer Institute for Student Success, Research, and Economic Development through the Kummer Innovation and Entrepreneurship Doctoral Fellowship.}



\appendix



\section{Notation}
\label{sec:appendix:notation}

For the convenience of the reader, Table~\ref{table:appendix_notations} provides a summary of the comprehensive notation used in this paper.
\begin{table}[htb!]
    \centering
    \begin{tabular}{@{}ll@{}}
        \toprule
        Symbol & Meaning \\
        \midrule
        $\inputdimension{}$ & Dimension of the \textbf{input} function domain \\
        $\outputdimension{}$ & Dimension of the \textbf{output} function domain \\
        $\inputdomain{} \subseteq \R^{\inputdimension}$ & The domain of the \textbf{input} functions \\
        $\outputdomain{} \subseteq \R^{\outputdimension}$ & The domain of the \textbf{output} functions \\
        $\inputspace{} \subseteq L^2(\inputdomain)$ & Vector space of \textbf{input} functions \\
        $\outputspace{} \subseteq L^2(\outputdomain)$ & Vector space of \textbf{output} functions \\
        $u \in \inputspace$ & \textbf{Input} function \\
        $v \in \outputspace$ & \textbf{Output} function \\
        \midrule
        
        $m$ & Dimension of the \textbf{input} discretized space \\
        $n$ & Dimension of the \textbf{output} discretized space \\
        $\inputdiscretization{}: \inputspace{} \to \R^\inputdiscdimension$ & Discretization of the \textbf{input} space \\
        $\outputdiscretization{} : \outputspace{} \to \R^\outputdiscdimension$ & Discretization of the \textbf{output} space \\
        $\vec{u} = \inputdiscretization(u) \in \R^{\inputdiscdimension}$ & Discrete representation of \textbf{input} function \\
        $\vec{v} = \inputdiscretization(v) \in \R^{\outputdiscdimension}$ & Discrete representation of \textbf{output} function \\
        $\inputdiscspace$ & Set of all \textbf{input} discretizations \\
        $\outputdiscspace$ & Set of all \textbf{output} discretizations \\
        $\opdiscspace$ & Set of all \textbf{operator} discretizations \\
        $\trueop : \inputspace \to \outputspace$ & True operator that is to be learned \\
        $\model{\modelparam} : \inputspace \to \outputspace$ & Parametric model for $\trueop$ with parameters $\modelparam \in \R^P$ \\
        $\numericalmodel{\modelparam}{\inputdiscretization}{\outputdiscretization}{\cdot} : \R^\inputdiscdimension \to \R^\outputdiscdimension$ & The numerical implementation of $\model{\modelparam}{\cdot}$ \\
        \midrule
        
        $\inputlatentdim$ & Dimension of \textbf{input} latent space (before approximator) \\
        $\outputlatentdim$ & Dimension of \textbf{output} latent space (after approximator) \\
        $\encoder: \inputspace \to \R^{\inputlatentdim}$ & \textbf{Encoder} module \\
        $\approximator : \R^{\inputlatentdim} \to \R^{\outputlatentdim}$ & \textbf{Approximator} module \\
        $\reconstructor : \R^{\outputlatentdim} \to \outputspace$ & \textbf{Reconstructor} module \\
        $\encoderparams$ & \textbf{Encoder} parameters \\
        $\approxparams$ & \textbf{Approximator} parameters \\
        $\reconstructparams$ & \textbf{Reconstructor} parameters \\
        $\encodernet : \R^{\inputdimension} \to \R^{\inputlatentdim}$ & \textbf{Encoder} basis network \\
        $\approxnet : \R^{\inputlatentdim} \to \R^{\outputlatentdim}$ & \textbf{Approximator} neural network \\
        $\reconstructnet: \R^{\outputdimension} \to \R^{\outputlatentdim}$ & \textbf{Reconstructor} basis network \\
        \bottomrule
    \end{tabular}
    \vskip -1mm\caption{Summary of notations used in the paper.}
    \label{table:appendix_notations}
\end{table}

\section{Comparing discretizations and effective discretization sets}
\label{sec:appendix:compare}

In this section we formalize the notion of the ``size'' of an effective discretization set by introducing a partial order on such sets. Since effective discretization sets are subsets of $\opdiscspace$, a natural way to compare them is via set inclusion on the  power set $\mathcal{P}(\opdiscspace)$. 
However, directly using inclusion as the partial order  is not quite satisfactory because some operator discretizations are ``weaker" than (or equivalent to) others; that is, all their information is contained in other discretizations. 
For example, suppose that $(\inputdiscretization^1, \outputdiscretization) \in \disc_1$ with $\inputdiscretization^1(u) = u(x_1) \in \mathbb{R}^1$, and $(\inputdiscretization^2, \outputdiscretization) \in \disc_2$ with  $\inputdiscretization^2(u) = (u(x_1), u(x_2))\in \mathbb{R}^2$. 
Then $\inputdiscretization^1$ can be recovered exactly from $\inputdiscretization^2$, which makes $\inputdiscretization^1$ weaker in the sense that any numerical operator $\numericalmodelnaked{\modelparam}$ with effective discretization set $\disc_1$ can be extended to $\inputdiscretization^2$ by simply discarding the extra information that $\inputdiscretization^2$ provides (that is, $u(x_2)$). Nevertheless, it may be possible that $(\inputdiscretization^2, \outputdiscretization) \notin \disc_1$ under the definition of the effective discretization set. 

This difficulty can be addressed in two ways. First, we could require effective discretization sets to include all stronger and equivalent discretizations, thus allowing comparison by set inclusion. Second, we could allow arbitrary sets and move the work of comparing properly to the definition of the partial order. 
We adopt the first approach for its simplicity.
To this end, we first define the strength of discretizations.

\begin{definition}[Strength of discretizations]
    \label{def:compare:disc_strength}
    Let $\inputdiscretization^1 : \inputspace \to \mathbb{R}^{\inputdiscdimension_1}$ and $\inputdiscretization^2 : \inputspace \to \mathbb{R}^{\inputdiscdimension_2}$ be two input discretizations. We say that $\inputdiscretization^1$ is \emph{weaker than} $\inputdiscretization^2$, and write as $\inputdiscretization^1 \le \inputdiscretization^2$, if there exists a mapping $\pi: \mathbb{R}^{\inputdiscdimension_2} \to \mathbb{R}^{\inputdiscdimension_1}$ such that
    \begin{equation}
        \inputdiscretization^1(u) = \pi(\inputdiscretization^2(u)), \qquad \text{for all } u\in \inputspace.
    \end{equation}
    That is, $\inputdiscretization^1 = \pi \circ \inputdiscretization^2$. 
\end{definition}
An analogous definition applies to output discretizations, as do all the subsequent statements and definitions.
This definition of strength does not quite induce a partial order on $\inputdiscspace$ because it fails to be antisymmetric; that is, we can have $\inputdiscretization^1 \le \inputdiscretization^2$ and $\inputdiscretization^2 \le \inputdiscretization^1$ but $\inputdiscretization^1 \ne \inputdiscretization^2$. For example, this is the case if $\inputdiscretization^1(u) = (u(x_1), u(x_2))$, and $\inputdiscretization^2(u) = (u(x_2), u(x_1))$.

\begin{proposition}
    \label{prop:compare:disc_strength}
    The strength relation is reflexive and transitive but not antisymmetric.
\end{proposition}
\begin{proof}
    Failure of antisymmetry was shown by the example above. 
    We now prove reflexivity and transitivity. Reflexivity follows by taking $\pi$ to be the identity mapping, which gives $\inputdiscretization^1 = \pi \circ \inputdiscretization^1$ and hence $\inputdiscretization^1 \le \inputdiscretization^1$.
    
For transitivity, suppose that $\inputdiscretization^1 \le \inputdiscretization^2$ and $\inputdiscretization^2 \le \inputdiscretization^3$. Then there exist mappings $\pi_2 : \mathbb{R}^{\inputdiscdimension_2} \to \mathbb{R}^{\inputdiscdimension_1}$ and $\pi_3 : \mathbb{R}^{\inputdiscdimension_3} \to \mathbb{R}^{\inputdiscdimension_2}$ such that
 $\inputdiscretization^1 = \pi_2\circ \inputdiscretization^2$, and $\inputdiscretization^2 = \pi_3 \circ \inputdiscretization^3$. 
Hence,
$\inputdiscretization^1 = (\pi_2\circ\pi_3)\circ \inputdiscretization^3$
which shows that $\inputdiscretization^1 \le \inputdiscretization^3$ by definition.
\end{proof}

We can now give the precise definition of effective discretization sets with the requirement of closure with respect to discretization strength.
\begin{definition}[Effective discretization set]
    \label{def:compare:effective_disc_set}
    Suppose a numerical operator learning model $\numericalmodelnaked{\modelparam}$ provides a valid approximation of an underlying operator model for every operator discretization
    $(\inputdiscretization,\outputdiscretization)\in \disc \subseteq \opdiscspace$. 
    We call $\disc$ the \emph{effective discretization set} of $\numericalmodelnaked{\modelparam}$.
    
    We require an effective discretization set $\disc$ to be \emph{closed with respect to strength}: if
    $(\inputdiscretization^1,\outputdiscretization^1)\in \disc$, and
    $(\inputdiscretization^2,\outputdiscretization^2)$ is a weaker operator discretization in the sense that
    $\inputdiscretization^2 \le \inputdiscretization^1$ and
    $\outputdiscretization^2 \le \outputdiscretization^1$, then
    $(\inputdiscretization^2,\outputdiscretization^2)\in \disc$.
\end{definition}
With the closure requirement of effective discretization sets, we may use set inclusion as a partial order on the set of all effective discretization sets. 
It avoids the issue that two effective discretization sets may not be comparable despite containing only discretizations that are comparable in the sense we have defined above. More precisely, we can characterize inclusion of effective discretization sets in terms of comparisons between their individual operator discretizations. 
\begin{proposition}
    Let $\numericalmodelnaked{\modelparam}^1$ and $\numericalmodelnaked{\modelparam}^2$ be two numerical operator learning models with effective discretization sets $\disc_1$ and $\disc_2$, respectively. Then $\disc_1 \subseteq \disc_2$ if and only if for any $(\inputdiscretization^1, \outputdiscretization^1) \in \disc_1$ there exists $(\inputdiscretization^2, \outputdiscretization^2) \in \disc_2$ such that $\inputdiscretization^1 \le \inputdiscretization^2$ and $\outputdiscretization^1 \le \outputdiscretization^2$.
\end{proposition}
\begin{proof}
    Suppose that $\disc_1 \subseteq \disc_2$. Then for any $(\inputdiscretization^1, \outputdiscretization^1) \in \disc_1$ we have $(\inputdiscretization^1, \outputdiscretization^1) \in \disc_2$, and $\inputdiscretization^1 \le \inputdiscretization^1$ and $\outputdiscretization^1 \le \outputdiscretization^1$.
    
    Conversely, suppose that for any $(\inputdiscretization^1, \outputdiscretization^1) \in \disc_1$ there exists $(\inputdiscretization^2, \outputdiscretization^2) \in \disc_2$ such that $\inputdiscretization^1 \le \inputdiscretization^2$ and $\outputdiscretization^1 \le \outputdiscretization^2$. 
    By closure of $\disc_2$ with respect to strength, it follows that $(\inputdiscretization^1, \outputdiscretization^1) \in \disc_2$.  Thus, $\disc_1 \subseteq \disc_2$.
\end{proof}
This alternative characterization shows that inclusion defines a natural comparison between discretization sets. Moreover, it offers a practical and rigorous criterion for determining whether one set is larger than another or whether the two are incomparable.
}

\section{Training details}
\label{sec:appendix:training}
Here we provide details related to the numerical experiments that were omitted from Section 4. All training code was implemented in PyTorch \citep{paszke_2019} and run on an NVIDIA GeForce RTX 4070 Ti GPU with 12GB of video memory. 

\smallskip
\noindent\textbf{Model architectures: }
As mentioned in Section 4, we use the same model architecture for all training runs of a particular PDE problem; however, we used slightly different architectures for each problem. In all cases, we implemented each neural network in our model with a multilayer perceptron with Leaky ReLU activation function with a slope of 0.03; network architectures are presented in Table \ref{table:appendix_architectures}.
\begin{table}[htb!]
    \centering
    \begin{tabular}{@{}llrrr@{}}
        \toprule
        Problem & Network & Input & Hidden Layers & Output \\
        \midrule
        Fractional Poisson & $\encodernet$ & 290 & $128 \times 3$ & 96 \\
        & $\approxnet$ & 96 & $256 \times 2$ & 96 \\
        & $\reconstructnet$ & 290 & $128 \times 3$ & 96 \\
        \midrule

         Darcy flow &  $\encodernet$ &  74 &  $128 \times 3$ &  96 \\
        &  $\approxnet$ &  96 &  $256 \times 2$ & 96 \\
        & $\reconstructnet$  &  74 &  $128 \times 3$ &  96 \\
        \midrule
        
        Burgers & $\encodernet$ & 26 & $96 \times 3$ & 18 \\
        & $\approxnet$ & 18 & $128 \times 4$ & 18 \\
        & $\reconstructnet$ & 26 & $96 \times 3$ & 18 \\
        \midrule
        
        Navier--Stokes & $\encodernet$ & 202 & $128 \times 3$ & 96 \\
        & $\approxnet$ & 96 & $256 \times 4$ & 96 \\
        & $\reconstructnet$ & 202 & $128 \times 3$ & 96 \\
        \bottomrule
    \end{tabular}
    \vspace{-1mm}\caption{All neural network architectures used in our models (number of neurons in each layer; for hidden layers, number of neurons $\times$ number of layers).}
    \label{table:appendix_architectures}
    \vspace{-4mm}
\end{table} 
Additionally, we use Fourier feature expansion following a similar implementation to that of Lu et al.\ for DeepONet \citep{lu_2022}. The number of modes in each spatial direction are given in Table \ref{table:appendix_feature_expansion}; the need to take these extra features as input is why the encoder and reconstructor networks have a different number of input neurons from the expected $\inputdimension$ and $\outputdimension$.

\begin{table}[htb!]
    \centering
    \begin{tabular}{@{}llr@{}}
        \toprule
        Problem & Network & Modes \\
        \midrule
        Fractional Poisson & $\encodernet$ & 12 \\
        & $\reconstructnet$ & 12 \\
        \midrule
        
         Darcy flow &  $\encodernet$ &  6 \\
        & $\reconstructnet$ &  6 \\
        \midrule
        
        Burgers & $\encodernet$ & 12 \\
        & $\reconstructnet$ & 12 \\
        \midrule
        
        Navier--Stokes & $\encodernet$ & 10 \\
        & $\reconstructnet$ & 10 \\
        \bottomrule
    \end{tabular}
    \vspace{-1mm}
    \caption{Number of Fourier feature expansion modes in all models.}
    \label{table:appendix_feature_expansion}
    \vspace{-1mm}
\end{table}

\smallskip
\noindent\textbf{Training hyperparameters: } 
To train our models we used the Adam \citep{kingma_2015} optimizer with momentum parameters $\beta_1 = 0.9$ and $\beta_2 = 0.999$ (defaults in PyTorch \citep{paszke_2019}) and a batch size of 10 for all problems. We chose different values of the following hyperparameters for each problem: the initial learning rate $\eta_0$, the geometric learning rate decay factor $\gamma$, and the total number of training epochs $N_e$. The values used on each problem are presented in Table \ref{table:appendix_hyperparameters}. These values were chosen through preliminary experiments to balance performance and training cost.

\begin{table}[htb!]
    \centering
    \begin{tabular}{@{}lrrr@{}}
        \toprule
         Problem & $\eta_0$ & $\gamma$ & $N_e$ \\
         \midrule
         Fractional Poisson & 0.003 & 0.9995 & 2000 \\
         Darcy flow &  0.003 &  0.9995 &  2000 \\
         Burgers & 0.005 & 0.9970 & 1000 \\
         Navier--Stokes & 0.003 & 0.9990 & 750 \\
         \bottomrule
    \end{tabular}
    \vspace{-1mm}
    \caption{Training hyperparameters by problem. Recall that the total training data size is $N=2048$ for  the fractional Poisson problem, the Darcy flow problem and the Burgers problem and $N=4096$ for the Navier--Stokes problem.} 
    \label{table:appendix_hyperparameters}
    \vspace{-6mm}
\end{table}
\clearpage



\begin{thebibliography}{43}
\providecommand{\natexlab}[1]{#1}
\providecommand{\url}[1]{\texttt{#1}}
\expandafter\ifx\csname urlstyle\endcsname\relax
  \providecommand{\doi}[1]{doi: #1}\else
  \providecommand{\doi}{doi: \begingroup \urlstyle{rm}\Url}\fi

\bibitem[Alkin et~al.(2024)Alkin, F{\"u}rst, Schmid, Gruber, Holzleitner, and
  Brandstetter]{alkin_2024}
B.~Alkin, A.~F{\"u}rst, S.~L. Schmid, L.~Gruber, M.~Holzleitner, and
  J.~Brandstetter.
\newblock Universal physics transformers: A framework for efficiently scaling
  neural operators.
\newblock In \emph{the 38th Annual Conference on Neural Information Processing
  Systems}, 2024.

\bibitem[Bartolucci et~al.(2023)Bartolucci, {de Bezenac}, Raonic, Molinaro,
  Mishra, and Alaifari]{bartolucci_2023}
F.~Bartolucci, E.~{de Bezenac}, B.~Raonic, R.~Molinaro, S.~Mishra, and
  R.~Alaifari.
\newblock Representation equivalent neural operators: A framework for
  alias-free operator learning.
\newblock In \emph{the 37th Conference on Neural Information Processing
  Systems}, 2023.

\bibitem[Berner et~al.(2022)Berner, Grohs, Kutyniok, and Petersen]{berner_2022}
J.~Berner, P.~Grohs, G.~Kutyniok, and P.~Petersen.
\newblock {T}he {M}odern {M}athematics of {D}eep {L}earning.
\newblock In \emph{Mathematical Aspects of Deep Learning}, pages 1--111.
  Cambridge University Press, 2022.

\bibitem[Bhattacharya et~al.(2021)Bhattacharya, Hosseini, Kovachki, and
  Stuart]{bhattacharya_2021}
K.~Bhattacharya, B.~Hosseini, N.~B. Kovachki, and A.~M. Stuart.
\newblock Model reduction and neural networks for parametric {{PDEs}}.
\newblock \emph{SMAI J. Comp. Math.}, 7:\penalty0 121--157, 2021.

\bibitem[Cao(2021)]{cao_2021}
S.~Cao.
\newblock Choose a transformer: {{Fourier}} or {{Galerkin}}.
\newblock In \emph{Advances in Neural Information Processing Systems},
  volume~34, 2021.

\bibitem[De et~al.(2023)De, Reynolds, Hassanaly, King, and Doostan]{de_2023}
S.~De, M.~Reynolds, M.~Hassanaly, R.~N. King, and A.~Doostan.
\newblock Bi-fidelity modeling of uncertain and partially unknown systems using
  {{DeepONets}}.
\newblock \emph{Comput. Mech.}, 71\penalty0 (6):\penalty0 1251--1267, 2023.

\bibitem[Duo et~al.(2018)Duo, van Wyk, and Zhang]{Duo2018}
S.~Duo, H.~W. van Wyk, and Y.~Zhang.
\newblock A novel and accurate finite difference method for the fractional
  {L}aplacian and the fractional {P}oisson problem.
\newblock \emph{J. Comput. Phys.}, 355:\penalty0 233--252, 2018.

\bibitem[{Fern{\'a}ndez-Godino}(2023)]{fernandez-godino_2023}
M.~G. {Fern{\'a}ndez-Godino}.
\newblock Review of multi-fidelity models.
\newblock \emph{Adv. Comput. Sci. Eng.}, 1\penalty0 (4):\penalty0 351--400,
  2023.

\bibitem[Fonseca et~al.(2023)Fonseca, Zappala, Caro, and Dijk]{fonseca_2023}
A.~H. D.~O. Fonseca, E.~Zappala, J.~O. Caro, and D.~V. Dijk.
\newblock Continuous spatiotemporal transformer.
\newblock In \emph{Proceedings of the 40th {{International Conference}} on
  {{Machine Learning}}}, pages 7343--7365, 2023.

\bibitem[Goswami et~al.(2023)Goswami, Bora, Yu, and Karniadakis]{goswami_2022}
S.~Goswami, A.~Bora, Y.~Yu, and G.~E. Karniadakis.
\newblock Physics-{{Informed Deep Neural Operator Networks}}.
\newblock In \emph{Machine Learning in Modeling and Simulation: Methods and
  Applications}, pages 219--254. Springer International Publishing, 2023.

\bibitem[Gustafsson and McBain(2020)]{skfem2020}
T.~Gustafsson and G.~D. McBain.
\newblock scikit-fem: A {P}ython package for finite element assembly.
\newblock \emph{J. Open Source Softw.}, 5\penalty0 (52):\penalty0
  2369, 2020.

\bibitem[Hao et~al.(2023)Hao, Wang, Su, Ying, Dong, Liu, Cheng, Song, and
  Zhu]{hao_2023}
Z.~Hao, Z.~Wang, H.~Su, C.~Ying, Y.~Dong, S.~Liu, Z.~Cheng, J.~Song, and
  J.~Zhu.
\newblock {{GNOT}}: A general neural operator transformer for operator
  learning.
\newblock In \emph{Proceedings of the 40th {{International Conference}} on
  {{Machine Learning}}}, pages 12556--12569, 2023.

 \bibitem[Herrmann et~al.(2024)Herrmann, Schwab, and Zech]{herrmann_2024}
L.~Herrmann, C.~Schwab, and J.~Zech.
\newblock Neural and spectral operator surrogates: Unified construction and
  expression rate bounds, 
\newblock \emph{Adv. Comput. Math.}, 50(4):72, 2024.

\bibitem[Hesthaven and Ubbiali(2018)]{hesthaven_2018}
J.~Hesthaven and S.~Ubbiali.
\newblock Non-intrusive reduced order modeling of nonlinear problems using
  neural networks.
\newblock \emph{J. Comput. Phys.}, 363:\penalty0 55--78, 2018.

\bibitem[Howard et~al.(2023)Howard, Perego, Karniadakis, and
  Stinis]{howard_2023}
A.~A. Howard, M.~Perego, G.~E. Karniadakis, and P.~Stinis.
\newblock Multifidelity deep operator networks for data-driven and
  physics-informed problems.
\newblock \emph{J. Comput. Phys.}, 493:\penalty0 112462, 2023.

\bibitem[Ingebrand et~al.(2025)Ingebrand, Thorpe, Goswami, Kumar, and Topcu]{ingebrand_2025}
T.~Ingebrand, A.~J. Thorpe, S.~Goswami, K.~Kumar, and U.~Topcu.
\newblock Basis-to-basis operator learning using function encoders.
\newblock \emph{Comput. Methods Appl. Mech. Engrg.}, 435:117646, 2025.

\bibitem[Kingma and Ba(2015)]{kingma_2015}
D.~P. Kingma and J.~Ba.
\newblock Adam: {A} method for stochastic optimization.
\newblock In \emph{the 3rd International Conference on Learning
  Representations}, 2015.

\bibitem[Kissas et~al.(2022)Kissas, Seidman, Guilhoto, Preciado, Pappas, and
  Perdikaris]{kissas_2022}
G.~Kissas, J.~Seidman, L.~F. Guilhoto, V.~M. Preciado, G.~J. Pappas, and
  P.~Perdikaris.
\newblock Learning operators with coupled attention.
\newblock \emph{J. Mach. Learn. Res.}, 23:\penalty0 1--63, 2022.

\bibitem[Kovachki et~al.(2024{\natexlab{a}})Kovachki, Li, Liu, Azizzadenesheli,
  Bhattacharya, Stuart, and Anandkumar]{kovachki_2024}
N.~Kovachki, Z.~Li, B.~Liu, K.~Azizzadenesheli, K.~Bhattacharya, A.~Stuart, and
  A.~Anandkumar.
\newblock Neural operator: Learning maps between function spaces with
  applications to {{PDEs}}.
\newblock \emph{J. Mach. Learn. Res.}, 24\penalty0 (1):\penalty0 1--97,
  2024{\natexlab{a}}.

\bibitem[Kovachki et~al.(2021)Kovachki, Lanthaler, and Mishra]{kovachki_2021}
N.~B. Kovachki, S.~Lanthaler, and S.~Mishra.
\newblock On universal approximation and error bounds for {{Fourier Neural
  Operators}}.
\newblock \emph{J. Mach. Learn. Res.}, 22:\penalty0 1--76, 2021.

\bibitem[Kovachki et~al.(2024{\natexlab{b}})Kovachki, Lanthaler, and
  Stuart]{kovachki_2024a}
N.~B. Kovachki, S.~Lanthaler, and A.~M. Stuart.
\newblock Chapter 9 - {{Operator Learning}}: {{Algorithms}} and {{Analysis}}.
\newblock In \emph{Numerical Analysis Meets Machine Learning}, volume~25 of
  \emph{Handbook of Numerical Analysis}, pages 419--467. 2024{\natexlab{b}}.

\bibitem[Lanthaler et~al.(2022)Lanthaler, Mishra, and
  Karniadakis]{lanthaler_2022}
S.~Lanthaler, S.~Mishra, and G.~E. Karniadakis.
\newblock Error estimates for {{DeepONets}}: A deep learning framework in
  infinite dimensions.
\newblock \emph{Trans. Math. Appl.}, 6\penalty0 (1):\penalty0 tnac001, 2022.

\bibitem[Li et~al.(2020)Li, Kovachki, Azizzadenesheli, Liu, Stuart,
  Bhattacharya, and Anandkumar]{li_2020}
Z.~Li, N.~Kovachki, K.~Azizzadenesheli, B.~Liu, A.~Stuart, K.~Bhattacharya, and
  A.~Anandkumar.
\newblock Multipole graph neural operator for parametric partial differential
  equations.
\newblock In \emph{Advances in Neural Information Processing Systems},
  volume~33, pages 6755--6766, 2020.

\bibitem[Li et~al.(2021)Li, Kovachki, Azizzadenesheli, {Liu}, Bhattacharya,
  Stuart, and Anandkumar]{li_2021a}
Z.~Li, N.~B. Kovachki, K.~Azizzadenesheli, B.~{Liu}, K.~Bhattacharya,
  A.~Stuart, and A.~Anandkumar.
\newblock Fourier neural operator for parametric partial differential
  equations.
\newblock In \emph{International Conference on Learning Representations}, 2021.

\bibitem[Li et~al.(2022)Li, Peng, Yuan, and Wang]{li_2022c}
Z.~Li, W.~Peng, Z.~Yuan, and J.~Wang.
\newblock Fourier neural operator approach to large eddy simulation of
  three-dimensional turbulence.
\newblock \emph{Theor. Appl. Mech. Lett.}, 12\penalty0 (6):\penalty0 100389,
  2022.

\bibitem[Li et~al.(2023{\natexlab{a}})Li, Kovachki, Choy, Li, Kossaifi, Otta,
  Nabian, Stadler, Hundt, Azizzadenesheli, and Anandkumar]{li_2023b}
Z.~Li, N.~B. Kovachki, C.~Choy, B.~Li, J.~Kossaifi, S.~P. Otta, M.~A. Nabian,
  M.~Stadler, C.~Hundt, K.~Azizzadenesheli, and A.~Anandkumar.
\newblock Geometry-informed neural operator for large-scale {{3D PDEs}}.
\newblock In \emph{the 37th Conference on Neural Information Processing
  Systems}, 2023{\natexlab{a}}.

\bibitem[Li et~al.(2023{\natexlab{b}})Li, Meidani, and Farimani]{li_2023}
Z.~Li, K.~Meidani, and A.~B. Farimani.
\newblock Transformer for partial differential equations' operator learning.
\newblock \emph{Trans. Mach. Learn.}, 2023{\natexlab{b}}.

\bibitem[Li et~al.(2023{\natexlab{c}})Li, Shu, and Farimani]{li_2023a}
Z.~Li, D.~Shu, and A.~B. Farimani.
\newblock Scalable transformer for {{PDE}} surrogate modeling.
\newblock In \emph{the 37th Conference on Neural Information Processing
  Systems}, 2023{\natexlab{c}}.

\bibitem[Lu et~al.(2021)Lu, Jin, and Karniadakis]{lu_2021}
L.~Lu, P.~Jin, and G.~E. Karniadakis.
\newblock {{DeepONet}}: {{Learning}} nonlinear operators for identifying
  differential equations based on the universal approximation theorem of
  operators.
\newblock \emph{Nat. Mach. Intell.}, 3\penalty0 (3):\penalty0 218--229, 2021.

\bibitem[Lu et~al.(2022{\natexlab{a}})Lu, Meng, Cai, Mao, Goswami, Zhang, and
  Karniadakis]{lu_2022}
L.~Lu, X.~Meng, S.~Cai, Z.~Mao, S.~Goswami, Z.~Zhang, and G.~E. Karniadakis.
\newblock A comprehensive and fair comparison of two neural operators (with
  practical extensions) based on {{FAIR}} data.
\newblock \emph{Comput. Methods Appl. Mech. Eng.}, 393:\penalty0 114778,
  2022{\natexlab{a}}.

\bibitem[Lu et~al.(2022{\natexlab{b}})Lu, Pestourie, Johnson, and
  Romano]{lu_2022a}
L.~Lu, R.~Pestourie, S.~G. Johnson, and G.~Romano.
\newblock Multifidelity deep neural operators for efficient learning of partial
  differential equations with application to fast inverse design of nanoscale
  heat transport.
\newblock \emph{Phys. Rev. Res.}, 4\penalty0 (2):\penalty0 023210,
  2022{\natexlab{b}}.

\bibitem[Ovadia et~al.(2024)Ovadia, Kahana, Stinis, Turkel, Givoli, and
  Karniadakis]{ovadia_2024}
O.~Ovadia, A.~Kahana, P.~Stinis, E.~Turkel, D.~Givoli, and G.~E. Karniadakis.
\newblock {{ViTO}}: Vision transformer-operator.
\newblock \emph{Comput. Methods Appl. Mech. Eng.}, 428:\penalty0 117109, 2024.

\bibitem[Paszke et~al.(2019)Paszke, Gross, Massa, Lerer, Bradbury, Chanan,
  Killeen, Lin, Gimelshein, Antiga, Desmaison, Kopf, Yang, DeVito, Raison,
  Tejani, Chilamkurthy, Steiner, Fang, Bai, and Chintala]{paszke_2019}
A.~Paszke, S.~Gross, F.~Massa, A.~Lerer, J.~Bradbury, G.~Chanan, T.~Killeen,
  Z.~Lin, N.~Gimelshein, L.~Antiga, A.~Desmaison, A.~Kopf, E.~Yang, Z.~DeVito,
  M.~Raison, A.~Tejani, S.~Chilamkurthy, B.~Steiner, L.~Fang, J.~Bai, and
  S.~Chintala.
\newblock Pytorch: An imperative style, high-performance deep learning library.
\newblock In \emph{Advances in Neural Information Processing Systems}, volume~32, pages 8024--8035. 2019.

\bibitem[Pinkus(1999)]{pinkus_1999}
A.~Pinkus.
\newblock Approximation theory of the {{MLP}} model in neural networks.
\newblock \emph{Acta Numerica}, 8:\penalty0 143--195, 1999.

\bibitem[Prasthofer et~al.(2022)Prasthofer, De~Ryck, and
  Mishra]{prasthofer_2022}
M.~Prasthofer, T.~De~Ryck, and S.~Mishra.
\newblock Variable-input deep operator networks, 
\emph{arXiv:2205.11404}, 2022.

\bibitem[Raonic et~al.(2023)Raonic, Molinaro, Ryck, Rohner, Bartolucci,
  Alaifari, Mishra, and {de Bezenac}]{raonic_2023}
B.~Raonic, R.~Molinaro, T.~D. Ryck, T.~Rohner, F.~Bartolucci, R.~Alaifari,
  S.~Mishra, and E.~{de Bezenac}.
\newblock Convolutional neural operators for robust and accurate learning of
  {{PDEs}}.
\newblock In \emph{the 37th Conference on Neural Information Processing
  Systems}, 2023.

\bibitem[Ros-Oton and Serra(2014)]{Ros-Oton2014}
X.~Ros-Oton and J.~Serra.
\newblock The {D}irichlet problem for the fractional {L}aplacian: regularity up
  to the boundary.
\newblock \emph{J. Math. Pures Appl. (9)}, 101\penalty0 (3):\penalty0 275--302,
  2014.

{\cn 
\bibitem[Sakarvadia et~al.(2026)Sakarvadia, Hegazy, Totounferoush, Chard,
  Yang, Foster, and Mahoney]{sakarvadia_2026}
M.~Sakarvadia, K.~Hegazy, A.~Totounferoush, K.~Chard, Y.~Yang, I.~Foster,
  and M.~W. Mahoney.
\newblock The false promise of zero-shot super-resolution in machine-learned
  operators.
\newblock In \emph{International Conference on Learning Representations}, 2026.}

\bibitem[Schwab et~al.(2023)Schwab, Stein, and Zech]{schwab_2023}
C.~Schwab, A.~Stein, and J.~Zech.
\newblock Deep {{Operator Network Approximation Rates}} for {{Lipschitz
  Operators}}, 
\newblock \emph{arXiv:2307.09835v1}, 2023.

{\cn \bibitem[Subedi and Tewari(2026)]{subedi_2026}
U.~Subedi and A.~Tewari.
\newblock Is zero-shot super-resolution possible in operator learning?
\newblock \emph{arXiv:2606.00296}, 2026.}

\bibitem[Sun et~al.(2023)Sun, Moya, Lin, and Yue]{sun_2023a}
Y.~Sun, C.~Moya, G.~Lin, and M.~Yue.
\newblock {{DeepGraphONet}}: A deep graph operator network to learn and
  zero-shot transfer the dynamic response of networked systems.
\newblock \emph{IEEE Syst. J.}, 17\penalty0 (3):\penalty0 4360--4370, 2023.

\bibitem[Tanyu et~al.(2023)Tanyu, Ning, Freudenberg, Heilenk{\"o}tter,
  Rademacher, Iben, and Maass]{tanyu_2023}
D.~N. Tanyu, J.~Ning, T.~Freudenberg, N.~Heilenk{\"o}tter, A.~Rademacher,
  U.~Iben, and P.~Maass.
\newblock Deep learning methods for partial differential equations and related
  parameter identification problems.
\newblock \emph{Inverse Probl.}, 39\penalty0 (10):\penalty0 103001, 2023.

\bibitem[Thakur et~al.(2022)Thakur, Tripura, and Chakraborty]{thakur_2022}
A.~Thakur, T.~Tripura, and S.~Chakraborty.
\newblock Multi-fidelity wavelet neural operator with application to
  uncertainty quantification.
\newblock \emph{arXiv:2208.05606v2},  2022.

\bibitem[Wen et~al.(2022)Wen, Li, Azizzadenesheli, Anandkumar, and
  Benson]{wen_2022}
G.~Wen, Z.~Li, K.~Azizzadenesheli, A.~Anandkumar, and S.~M. Benson.
\newblock U-{{FNO}} -- {{An}} enhanced {{Fourier}} neural operator-based
  deep-learning model for multiphase flow.
\newblock \emph{Adv. Water Resour.}, 163:\penalty0 104180, 2022.

\bibitem[Zhou and Zhang(2024)]{Zhou2024}
S.~Zhou and Y.~Zhang.
\newblock A novel and simple spectral method for nonlocal {PDE}s with the
  fractional {L}aplacian.
\newblock \emph{Comput. Math. Appl.}, 168:\penalty0 133--147, 2024.

\end{thebibliography}

\end{document}